\documentclass[letterpaper]{article} % DO NOT CHANGE THIS
\usepackage[preprint]{aaai2027}  % DO NOT CHANGE THIS
\usepackage[hyphens]{url}  % DO NOT CHANGE THIS
\usepackage{graphicx} % DO NOT CHANGE THIS
\usepackage{natbib}  % DO NOT CHANGE THIS AND DO NOT ADD ANY OPTIONS TO IT
\usepackage{caption} % DO NOT CHANGE THIS AND DO NOT ADD ANY OPTIONS TO IT
\usepackage{subcaption}
\usepackage{algorithm}
\usepackage{algorithmic}

\usepackage{booktabs} % for professional tables
\usepackage{mathtools}
\usepackage{xparse}
\usepackage{amsfonts,amsthm,amssymb}
\usepackage{dsfont}
\usepackage{enumitem}
\newcommand{\Cq}{C_\textnormal{quad}}

\usepackage[capitalize,noabbrev]{cleveref}

\allowdisplaybreaks

\newtheorem{theorem}{Theorem}[section]

\newtheorem{lemma}[theorem]{Lemma}
\newtheorem{corollary}[theorem]{Corollary}
\newtheorem{remark}[theorem]{Remark}

\newtheorem{proposition}[theorem]{Proposition}

\newtheorem{assumption}{Assumption}

\newcommand{\epsp}{\epsilon'}

\newcommand{\VL}{V_\lambda}
\newcommand{\gammaf}{\gamma_\textnormal{eff}}

\newcommand{\diag}{\textnormal{diag}}
\newcommand{\SA}{|\cS||\cA|}

\newcommand{\px}{\pi_x}

\newcommand{\bF}{\bar{F}}
\newcommand{\bFlin}{\bar{F}_\textnormal{lin}}

\newcommand{\bI}{\mathbb{I}}

\newcommand{\cC}{\mathcal{C}}

\newcommand{\bDelta}{\bar{\Delta}}

\newcommand{\TV}{\textnormal{TV}}
\newcommand{\cO}{\mathcal{O}}
\newcommand{\wcO}{\widetilde{\cO}}

\newcommand{\p}[2]{#1 \oslash #2}

\newcommand{\bP}{\mathbb{P}}
\newcommand{\tF}{\widetilde{F}}
\newcommand{\tG}{\widetilde{G}}

\newcommand{\nstr}{n_*}
\newcommand{\taun}{\tau_n}

\newcommand{\bDeltainit}{\bar{\Delta}^\textnormal{init}}
\newcommand{\bDeltatrans}{\bar{\Delta}^\textnormal{trans}}
\newcommand{\bDeltalin}{\bar{\Delta}^\textnormal{noise}}
\newcommand{\bDeltanlin}{\bar{\Delta}^\textnormal{nonlin}}

\newcommand{\e}{\mathrm{e}}

\newcommand{\ceil}[1]{\lceil #1 \rceil}
\newcommand{\floor}[1]{\lfloor #1 \rfloor}

\NewDocumentCommand{\dx}{m g}{%
  \IfNoValueTF{#2}
    {d(#1)}
    {d(#1,#2)}
}
\NewDocumentCommand{\dxp}{m g}{%
  \IfNoValueTF{#2}
    {\|#1\|_{\times, \,p}}
    {\|#1 - #2\|_{\times, \,p}}}

\newcommand{\bE}{\mathbb{E}}
\newcommand{\bR}{\mathbb{R}}

\newcommand{\bRp}{\bR_{++}}
\newcommand{\Cm}{C_{\max}}

\newcommand{\Xstr}{X^*}
\newcommand{\xstr}{x^*}
\newcommand{\Qstr}{Q^*}

\newcommand{\cA}{\mathcal{A}}

\newcommand{\cF}{\mathcal{F}}
\newcommand{\cK}{\mathcal{K}}
\newcommand{\cM}{\mathcal{M}}
\newcommand{\cP}{\mathcal{P}}

\newcommand{\cS}{\mathcal{S}}

\newcommand{\Rm}{R_\textnormal{max}}

\newcommand{\hF}{\widehat{F}}
\newcommand{\hG}{\widehat{G}}

\newcommand{\xp}{x_\pi}
\newcommand{\piS}{\pi^*}

\newcommand{\Gbeta}{\Gamma^{(\beta)}}
\newcommand{\Galpha}{\Gamma^{(\alpha)}}

\newcommand{\ones}{\mathds{1}}

\newcommand{\SMPol}{\Pi_{\rm SM}}
\newcommand{\MPol}{\Pi_{\rm M}}

\DeclareMathOperator*{\argmin}{arg\,min}

\usepackage{newfloat}
\usepackage{listings}
\DeclareCaptionStyle{ruled}{labelfont=normalfont,labelsep=colon,strut=off} % DO NOT CHANGE THIS
\floatstyle{ruled}
\newfloat{listing}{tb}{lst}{}
\floatname{listing}{Listing}

\usepackage{booktabs}

\title{Finite-Time Analysis of Discounted Exponential-Utility Reinforcement Learning}
\author{
    Ankur Naskar\textsuperscript{\rm 1}, Vivek T A\textsuperscript{\rm 1}, Aditya Kumar\textsuperscript{\rm 1}, Gugan Thoppe\textsuperscript{\rm 1}, Prashanth L. A.\textsuperscript{\rm2},\\
}
\affiliations{
    \textsuperscript{\rm 1}Indian Institute of Science, Bengaluru, India\\
    \textsuperscript{\rm 2}Indian Institute of Technology Madras, Chennai, India\\
    \texttt{ankurnaskar@iisc.ac.in,    vivek.ta04@gmail.com, aditya3@iisc.ac.in, \\gthoppe@iisc.ac.in, prashla@cse.iitm.ac.in}
}

\begin{document}

\maketitle

\begin{abstract}
Discounted exponential utility provides a principled criterion for risk-sensitive sequential decision-making, but its nonlinear structure complicates reinforcement learning. A recent work \citep{thoppe2026reinforcement} addressed this difficulty by introducing a Bellman-compatible surrogate and two model-free fixed-point algorithms for optimizing it over stationary policies. However, their main convergence results are asymptotic in nature. In this work, we establish finite-time rates of $\wcO(1/\sqrt{n})$ for the aforementioned two algorithms under asynchronous Markovian sampling, where $n$ is the iteration index and $\wcO$ hides logarithmic expressions. Importantly, we employ parameter-free choices for the stepsize parameter to derive these rate results.  For the algorithmically simpler one-timescale method, the main challenge is that its update equation is not directly aligned with the contraction geometry of its underlying power-law operator. We overcome this mismatch by exploiting the boundedness of the iterates, the monotonicity, and the homogeneity of the operator to get a local psuedo-contraction property of the relative-error dynamics. We then use a Moreau-envelope-based Lyapunov function and Polyak--Ruppert averaging to obtain the stated  convergence rate with parameter-free stepsizes. For the two-timescale method, the main challenge is to control a tracking error on the faster timescale. These results provide the first finite-time guarantees for model-free discounted exponential-utility reinforcement learning.
\end{abstract}
%


\section{Introduction}
\label{s: introduction}
Exponential utility is a classical criterion for risk-sensitive sequential decision-making, accounting for uncertainty and unfavorable outcomes beyond the standard expected-reward objective \cite{Howard1972}. While reinforcement-learning (RL) methods for exponential utility have been studied extensively in the average-reward setting, see \cite{borkar2010learning,biswas_ergodic_2023} for survey articles, the discounted case is structurally more difficult and has received much less attention in the literature. A direct treatment of the discounted objective yields a dynamic-programming recursion in which the effective risk-sensitivity parameter varies across stages, see \cite[Theorem 4]{chung1987discounted}. To recover a standard time-homogeneous Bellman formulation, \citet{porteus1975optimality} studied an alternative inter-stage-consistent recursive criterion that keeps the risk-sensitivity parameter fixed across stages. However, the relationship between this criterion and the direct discounted exponential-utility objective was not quantified, and no model-free algorithms were available for learning its solution from sampled transitions.

A recent work \citep{thoppe2026reinforcement} addressed both these issues. First, it showed that the direct and Porteus's alternative criteria can lead to different optimal policies in general, but there is only a $\cO(1 -\gamma)$ bound between the optimal value functions when the MDP has an absorbing state and every stationary policy is proper. Second, it proposed one- and two-timescale model-free algorithms for optimizing Porteus's criterion over stationary policies and established their almost-sure convergence. That work, though, provides finite-time guarantees for the two timescale algorithm only in IID or generative sampling scenario, and for the one timescale algorithm only in the scalar parameter case. 

Our work addresses the above research gaps by first extending the one-timescale analysis to the general vector case and then providing bounds under Markovian sampling for both algorithms. Our main contributions are as follows. 
\begin{enumerate}
    \item \textbf{One-timescale algorithm}: We establish a $\wcO(1/\sqrt{n})$  finite-time rate for the vector-version of the one-timescale algorithm under both synchronous and asynchronous Markovian sampling, where $n$ is the number of iterations. Importantly, we derive this rate using only parameter-free stepsizes. The key challenge is that the algorithm's additive update is not aligned with the logarithmic contraction geometry of the underlying power-law operator, preventing standard fixed-point arguments. We overcome this by using boundedness, monotonicity, and homogeneity to derive a local pseudo-contraction, and then combine a Moreau-envelope Lyapunov analysis with Polyak--Ruppert averaging to obtain the stated rate.   

    \item \textbf{Two-timescale algorithm}. Under asynchronous Markovian sampling, we  derive a $\wcO(1/\sqrt{n})$ finite-time rate using parameter-free stepsizes, with a logarithmic correction on the faster timescale. The key challenges here include deriving a nonasymptotic tracking bound for the faster recursion, propagating this error into the slower recursion, and jointly controlling the resulting tracking, mixing, and contraction errors. 
\end{enumerate}

\paragraph{Related work.}
Risk-sensitive RL studies a broad range of criteria, including mean-variance \citep{markowitz1952portfolio}, quantiles, exponential utility \citep{whittle1990risk}, Conditional Value-at-Risk (CVaR) \citep{rockafellar2000optimization}, and cumulative prospect theory \citep{tversky1992advances}. This topic has attracted substantial attention; see, e.g., 
\citep{prashanth2015cumulative,huang2020stochastic,kose2021risk,agrawal24policy,thoppe24risk,borkar2001sensitivity,markowitz2023risk,borkar2010learning,prashanth2016variance,Mihatsch02RS,tamar2012policy,mihatsch_risk-senitive_2002,bauerle_markov_2024}. 

Within this literature, exponential utility is well studied in the average-reward setting. When the MDP dynamics are known, planning methods have been developed in \citep{Howard1972,whittle1990risk}. The model-free literature includes Q-learning \citep{borkar2002q}, policy evaluation with linear features \citep{Basu08LA}, and policy-gradient and actor--critic algorithms \citep{borkar2001sensitivity,moharrami2022policy,guin2026actorcritic}. In comparison, the discounted setting remains much less developed. Planning with discounted exponential utility has been studied in \citep{porteus1975optimality,jaquette1976utility,chung1987discounted}, whereas provably convergent RL algorithms were proposed only recently in \citep{thoppe2026reinforcement}. That work, however, does not provide general finite-time guarantees, leaving the convergence speed of its algorithms unresolved. Our results address this gap.

\section{Review: Discounted Exponential Utility}
\label{s: background}

Consider a finite Markov Decision Process (MDP)
$\cM = (\cS,\cA,\cP,r,\gamma)$, where $\cS$ and $\cA$ are the state and action spaces with $|\cS| = S$ and $|\cA| = A$, $\cP(\cdot \hspace{-0.5ex} \mid \hspace{-0.5ex} s,a)$ is the transition kernel, $r:\cS\times\cA\to\bR$ is the reward function, and $\gamma\in(0,1)$ is the discount factor. Let $\Pi$ be the class of all admissible, possibly randomized and
history-dependent policies, and  $\MPol \subseteq\Pi$ the class of possibly nonstationary Markov
policies. For $\pi \in \Pi$ and risk-sensitivity parameter $\theta>0$, the direct discounted exponential-utility state-action value is given by
\begin{equation}\label{e:exp.utility}
    \hspace{-1.25ex} X_\pi(s,a;\theta) \hspace{-0.25ex}  := \hspace{-0.25ex} \bE_\pi\bigg[e^{-\frac{\theta}{\gamma}\sum\limits_{t=0}^{\infty}\gamma^t r(s_t,a_t)} \hspace{-0.5ex} \biggm| \hspace{-0.5ex} (s_0, a_0) = (s, a)\bigg].
\end{equation}
Following \citet{chung1987discounted}, a finite-horizon
dynamic-programming argument, followed by passage to the
infinite-horizon limit, shows that 
\[
    \Xstr(s,a;\theta)
    :=
    \inf_{\pi\in\Pi}X_\pi(s,a;\theta)
    = \!
    \inf_{\pi\in \MPol}X_\pi(s,a;\theta)\,  \forall (s,a).
\]
That is, restricting $\Pi$ to $\MPol$ does not change the infimum. The same argument also yields the Bellman recursion
\begin{multline} 
\label{e:Xstr.Bellman.recursion} 
    \hspace{-0.5em} \Xstr(s,a;\theta) \\ \hspace{-0.5em}  = \exp\left(-\frac{\theta}{\gamma}r(s,a)\right) \sum_{s'}\cP(s'|s,a) \min_{a'}\Xstr(s',a';\gamma\theta). \hspace{-0.5em} 
\end{multline}
Since the $\Xstr$ term is evaluated at $\gamma \theta$ on the right, the effective risk-sensitivity parameter changes across stages. Hence, this Bellman recursion is not time-homogeneous, and an
optimal Markov policy need not be stationary.

To get a time-homogeneous formulation,
\citet{porteus1975optimality}---in simpler terms---considered \eqref{e:Xstr.Bellman.recursion} under log-transformation and replaced $\gamma\theta$ by $\theta$. This gives the alternative recursion 
\begin{multline}
\label{e:Qstr.Bellman.recursion}
    \Qstr(s,a; \theta) = r(s,a) \\
    - \frac{\gamma}{\theta}
    \ln\!\left[
    \sum_{s'}\cP(s'|s,a)
    \exp\!\left(
    -\theta\max_{a'}\Qstr(s',a'; \theta)
    \right)
    \right].
\end{multline}    
Now, by using the inverse log transformation
$
    \xstr(s, a; \theta)
    :=
    \exp\!\left(
        -\frac{\theta}{\gamma}
        \Qstr(s, a;\theta)
    \right), (s, a) \in \cS \times \cA,
$
\eqref{e:Qstr.Bellman.recursion} reads as
\begin{multline}
\label{e:xstr.Bellman.equation}
    \hspace{-1em} \xstr(s, a; \theta) \hspace{-0.25em}  = \hspace{-0.25em}  e^{-\frac{\theta}{\gamma}r(s,a)} \sum_{s'}\cP(s'|s,a) \hspace{-0.25em} \left[\min_{a'} \xstr(s',a'; \theta)\right]^\gamma \hspace{-0.5em}. \hspace{-0.75em}
\end{multline}
This equation is an inter-stage-consistent variant of \eqref{e:Xstr.Bellman.recursion}.

Let $\bRp^{SA}$ be the strictly positve orthant of $\bR^{SA}.$ Motivated by \eqref{e:xstr.Bellman.equation},
\citet{thoppe2026reinforcement} introduced the Bellman optimality
operator \(F:\bR_{++}^{SA}\to\bR_{++}^{SA}\), defined by
\begin{multline}
\label{e: operator.F}
    \hspace{-0.5em} F(x)(s,a) \\
    := 
    \exp\left(-\frac{\theta}{\gamma}r(s,a)\right) \sum_{s'}\cP(s'|s,a)\Big[\min_{a'}x(s',a')\Big]^\gamma. \hspace{-0.5em}
\end{multline}
They showed that \(F\) is a \(\gamma\)-contraction under the log-sup,
or Thompson, metric
\begin{equation}
    \label{e: log.sup.metric}
    d(x_1,x_2) := \|\ln x_1-\ln x_2\|_\infty,
\end{equation}
and therefore has a unique fixed point, namely the vector \(\xstr\). 

The inter-stage consistency of \eqref{e:xstr.Bellman.equation} motivates considering the set of stationary Markov policies, denoted by $\SMPol,$ and replacing the minimization
in \(F\) by a fixed policy. Accordingly, for each
\(\pi\in\SMPol\), \citet{thoppe2026reinforcement} considered the
policy-evaluation operator
\begin{equation}
\label{e:F.pi.operator}
    \hspace{-0.5em} F_\pi(x)(s,a)
    :=
    e^{-\frac{\theta}{\gamma}r(s,a)}
    \sum_{s',a'}
    \cP_\pi(s',a'|s,a)\,
    x^\gamma(s',a'),
\end{equation}
where $\cP_{\pi}(s', a'|s, a) = \cP(s'|s, a) \pi(a'|s').$ The operator \(F_\pi\) is also a \(\gamma\)-contraction in the Thompson
metric and therefore has a unique fixed point, denoted by
\(\bar{x}_\pi\). Moreover,
\[
    \xstr
    =
    \bar{x}_{\piS}
    =
    \inf_{\pi\in\SMPol}\bar{x}_\pi,
\]
where $\piS$ is a stationary policy greedy with respect to $\xstr.$ Thus, $\piS$ minimizes $\bar{x}_\pi$ componentwise over $\SMPol.$

Since stationary policies are easier to characterize,
implement, and learn, it is natural to ask
whether \(\piS\) also minimizes the direct criterion \(X_\pi\) over
\(\SMPol\). \citet{thoppe2026reinforcement} showed that this need not
hold in general: \(X_\pi\) and \(\bar{x}_\pi\) can rank stationary
policies differently. Nevertheless, for MDPs with absorbing states and
nonnegative rewards in which every stationary policy is proper, they
established that
\[
    0
    \leq
    \xstr(s,a)
    -
    \inf_{\pi\in\SMPol}X_\pi(s,a)
    \leq
    K(1-\gamma),
\]
where \(K\) uniformly bounds the expected absorption time over all
\(\pi\in\SMPol\) and all initial state-action pairs \((s,a)\).

Thus, for certain MDPs, $\bar{x}_{\pi}$ provides
a tractable surrogate for optimizing the direct exponential utility criterion $X_{\pi}.$

\begin{figure*}
    \centering
    
    % --- Row 1 ---
    \begin{subfigure}[b]{0.48\linewidth}
        \centering
        \includegraphics[width=\linewidth]{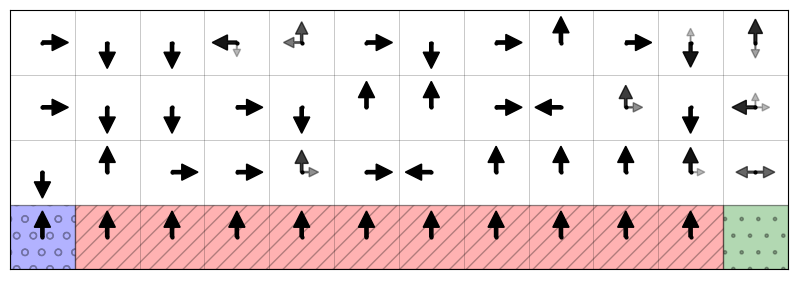}
        \caption{Risk-neutral Q-learning}
        \label{subfig:rn1}
    \end{subfigure}
    \hfill
    \begin{subfigure}[b]{0.48\linewidth}
        \centering
        \includegraphics[width=\linewidth]{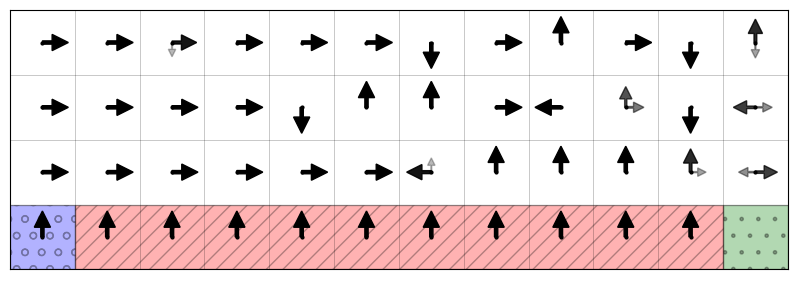}
        \caption{$\theta=10^{-4}$ (One-timescale)}
        \label{subfig:low1}
    \end{subfigure}
    
    \vspace{1em} % Adds a little vertical breathing room between rows
    
    % --- Row 2 ---
    \begin{subfigure}[b]{0.48\linewidth}
        \centering
        \includegraphics[width=\linewidth]{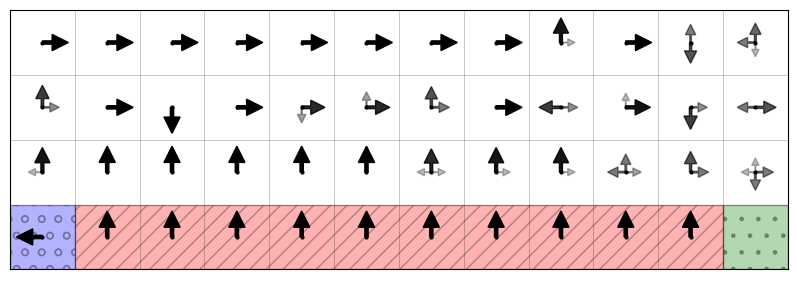}
        \caption{$\theta=10$ (Two-timescale)}
        \label{subfig:mid1}
    \end{subfigure}
    \hfill
    \begin{subfigure}[b]{0.48\linewidth}
        \centering
        \includegraphics[width=\linewidth]{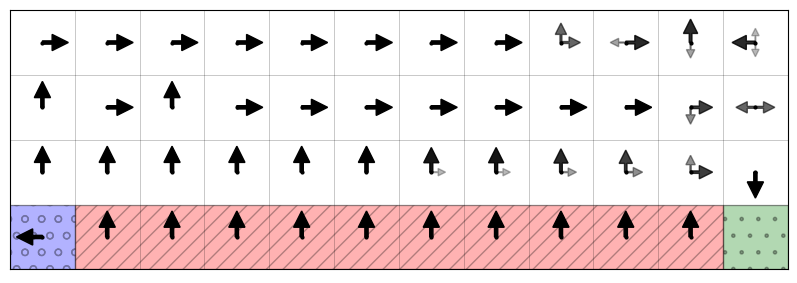}
        \caption{$\theta=10$ (One-timescale)}
        \label{subfig:large1}
    \end{subfigure}
       
    % --- Legend ---
    \begin{subfigure}[b]{0.48\linewidth}
        \centering
        \includegraphics[width=\linewidth]{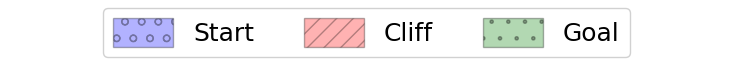}
        \label{subfig:legend}
    \end{subfigure}

    \caption{Visual representation of the Learned policies for risk-neutral Q-learning and the one- and two-timescale algorithms, given in \eqref{e: 1TS} and \eqref{e: 2TS}, respectively. The results are evaluated over \(10\) independent seeds. Arrow opacity indicates agreement across seeds on the greedy action at each state: darker arrows represent stronger consensus, while multiple lighter arrows indicate greater inter-seed variability. Overall, the learned policies become increasingly risk-averse as \(\theta\) increases.}
    \label{fig:cliffwalk-results}
\end{figure*}

\section{RL Algorithms for Exponential Utility}

The previous discussion motivates learning the tractable surrogate $\xstr$ from sampled transitions and examining whether its greedy policy exhibits meaningful risk-sensitive behavior.

\citet{thoppe2026reinforcement} gave two such algorithms: a
one-timescale algorithm for approximating \(\xstr\) directly and a two-timescale one for approximating its log-transformed counterpart
\(\Qstr\). To describe the latter, define 
\(T:\bR^{SA}\to\bR^{SA}\) by
\begin{multline}\label{e:operator.T}
    T(Q)(s,a)
    :=
    r(s,a)
    \\
    - \frac{\gamma}{\theta}
    \ln\!\left[
        \sum_{s'} \cP(s'|s,a)
        \exp\!\left(-\theta\max_{a'}Q(s',a')\right)
    \right].
\end{multline}
The operator \(T\) is a \(\gamma\)-contraction under
\(\|\cdot\|_\infty\) and therefore has a unique fixed point, namely
\(\Qstr\).

\subsection{One- and Two-timescale Update Rules}

We now briefly describe the two algorithms.

\paragraph{One-timescale algorithm.}
Since \(F(x)(s,a)\) is an expectation over the next state, it admits an
unbiased estimate from a single next-state sample. This motivates the
following synchronous update, in which all state-action components are
updated at every iteration:
\begin{equation}
\label{e:x.synchronous.update}
    x_{n+1}
    =
    x_n+\alpha_n[\widehat F_{n+1}-x_n],
\end{equation}
where
\[
    \widehat F_{n+1}(s,a)
    =
    \exp\!\left(-\frac{\theta}{\gamma}r(s,a)\right)
    \left[\min_{a'}x_n(s'_{s,a},a')\right]^\gamma,
\]
and \(s'_{s,a}\sim\cP(\cdot\mid s,a)\) is sampled for each
\((s,a)\in\cS\times\cA\).

Next, we consider the asynchronous variant, which learns from a sample path
and updates $x_n$ only on the currently visited state-action component. Specifically, let
$((s_n,a_n))_{n \geq 0}$ be generated from an arbitrary initial pair
\((s_0,a_0)\) as per $s_{n+1}\sim\cP(\cdot\mid s_n,a_n)$ and $a_{n+1}\sim\mu(\cdot\mid s_{n+1}),$ where \(\mu\) is a fixed behavior policy. Then the asynchronous update for $x_n$ at iteration $n \geq 0$ is
\begin{multline}\label{e: 1TS}
    x_{n+1} =
    x_n
    \\+ \alpha_n e_{s_n,a_n}\left[ \widehat{F}(x_n,s_n,a_n,s_{n+1})- x_n(s_n,a_n)\right],
\end{multline}
where \(e_{s,a}\) is the corresponding standard basis vector and
\[
    \widehat{F}(x,s,a,s')
    :=
    \exp\!\left(-\frac{\theta}{\gamma}r(s,a)\right)
    \left[\min_{a'}x(s',a')\right]^\gamma.
\]

\paragraph{Two-timescale algorithm.}
Since the expectation in \(T\) lies inside the logarithm, it cannot be
estimated unbiasedly from a single transition. The two-timescale method addresses this by using two coupled updates:
\begin{equation}\label{e: 2TS}
    \begin{aligned}
        & Q_{n+1} = Q_n + \beta_n\left[ -\frac{\gamma}{\theta}\ln g_n-Q_n \right],
        \\
        & g_{n+1} = g_n + \alpha_n e_{s_n,a_n}
        \bigg[
            \widehat{G}(Q_n,s_n,a_n,s_{n+1}) \\
        & \hspace{3em} - g_n(s_n,a_n) \bigg],
    \end{aligned}
\end{equation}

where
\[
    \widehat G(Q,s,a,s')
    :=
    \exp\!\left[
        -\frac{\theta}{\gamma}r(s,a)
        -
        \theta\max_{a'}Q(s',a')
    \right].
\]
With \(\beta_n/\alpha_n\to0\), \(g_n\) evolves on the faster timescale
and tracks the conditional expectation in \(T(Q_n)\). Hence,
\(-\frac{\gamma}{\theta}\ln g_n\) approximates \(T(Q_n)\), and the
\(Q_n\)-recursion resembles $Q_{n + 1} \approx Q_n + \beta_n [T (Q_n) - Q_n].$

\subsection{Evaluating Risk-Sensitive Behavior}
Since the two algorithms optimize the tractable surrogate \(\bar{x}_\pi\)
rather than \(X_\pi\) directly, it is not immediate that the induced
policies exhibit meaningful risk-sensitive behavior. As this question
was not examined in \citep{thoppe2026reinforcement}, we first conduct a numerical experiment to investigate it. 

\paragraph{Cliff walk: experimental setup.}
We consider a \(4\times 12\) cliff walk model with the start state at the bottom-left corner and the goal at the bottom-right corner; see Figure~\ref{fig:cliffwalk-results}. Each step incurs a $-1$ reward, while falling off the cliff results in a large negative reward of $-30$. We use a slippery variant in which the agent may fail to move in the intended direction. Specifically, the slip probability is \(0.02\) in the row nearest the cliff, \(0.8\) in the middle row, and \(0\) in the top row. Consequently, the row adjacent to the cliff provides the shortest but riskiest route, whereas the top row
provides the safest route. We set the discount factor $\gamma$ to $0.9$.

\paragraph{Training details.}
We compare risk-neutral Q-learning with the one- and two-timescale
methods, given in \eqref{e: 1TS} and \eqref{e: 2TS}, respectively, using \(\theta\in\{10^{-4},10\}\). We train all methods under
a fixed stochastic behavior policy over \(10\) independent seeds. Each
seed consists of \(500\) episodes, with every episode starting from the
same initial state and terminating upon reaching the goal or after
\(1000\) steps. In all of our runs, we set \(Q_0(s, a) =0\) and \(x_0(s, a) =1\) for all $(s, a).$ The step size for the one-timescale has been chosen to be $\alpha_n = (n+1)^{-0.51}$ and those for the two-timescale have been chosen to be $\alpha_n = (n+1)^{-0.51}$ and $\beta_n = (n+1)^{-0.52}.$

\paragraph{Experimental findings.}
Figure~\ref{fig:cliffwalk-results} shows the learned greedy policies.
Arrow opacity indicates agreement across the \(10\) seeds, with darker
arrows corresponding to stronger consensus. Risk-neutral Q-learning
favors the shortest route near the cliff. For \(\theta=10^{-4}\), the
one-timescale method behaves similarly to the risk-neutral method. For
\(\theta=10\), both surrogate-based methods favor safer routes farther
from the cliff, with the one-timescale method exhibiting less
variability across seeds than the two-timescale method. Thus, the
learned policies become increasingly risk-averse as \(\theta\)
increases.

\section{Main Results}
\label{s:main.results}
The asymptotic convergence of the algorithms in \eqref{e: 1TS} and \eqref{e: 2TS} was shown in \citep{thoppe2026reinforcement}. Here, we derive finite-time guarantees for the one-timescale algorithm under synchronous and asynchronous Markovian sampling, and for the two-timescale algorithm under asynchronous Markovian sampling. In both asynchronous cases, we assume that the samples are generated by a fixed behavior policy $\mu$. 

\subsection{Convergence rates: One-timescale }

We first highlight the main challenge in analyzing the one-timescale method. Its update resembles a standard stochastic fixed-point iteration, as is most transparent in the synchronous case. Specifically, \eqref{e:x.synchronous.update} can be rewritten as
\[
    x_{n+1} = x_n + \alpha_n[ F(x_n) - x_n + \zeta_{n+1}],
\]
where ${\zeta_{n+1}}$ is a martingale-difference sequence with $\zeta_{n + 1} = \hat{F}_{n + 1} - F(x_n).$ If $F$ were sup-norm contractive, as in risk-neutral Q-learning, one could smooth $\|x_n-\xstr\|_\infty^2$ using a suitable Moreau envelope, substitute the stochastic update, exploit the contraction of $F$, and unroll the resulting drift recursion. Here, however, $F$ is contractive only in the Thompson metric (see \eqref{e: log.sup.metric}). So, the above idea would be naturally compatible with a multiplicative update such as 
\[
    x_{n + 1} = x_n^{(1 - \alpha_n)}(\widehat{F}_{n + 1})^{\alpha_n},
\]
with all operations done compontentwise. However, taking a $\log$ to exploit the contraction, also places the sampling noise inside the $\log.$ The resulting error is no longer a martingale difference and, hence, introduces a systematic bias. The asymptotic analysis in \citep{thoppe2026reinforcement} established convergence through the relative error $x_n(s, a)/x^*(s, a) - 1$, making it the natural quantity for our finite-time analysis. The main difficulty now is that $F$'s contraction does not directly translate into contraction of this relative error, and our proof surmounts this difficulty in a novel fashion, see details below.

Let $x_0$ be the initial iterate and $\cK := [C_\ell, C_u]^{SA},$ where
\begin{equation}\label{e: Cl.Cu.defn}
\begin{aligned}
    C_\ell & :=  \min\left\{\exp\left(-\frac{\theta \Rm}{\gamma(1-\gamma)}\right), \min_{s,a} x_0(s,a) \right\}
    \\
    C_u & := \max\left\{ \exp\left(\frac{\theta \Rm}{\gamma(1-\gamma)}\right), \max_{s,a} x_0(s,a) \right\},
\end{aligned}
\end{equation}
and $\Rm := \max_{s, a} |r(s, a)| < \infty.$
Separately, let 
\begin{equation}\label{e: def.gammaf}
    \gammaf := \sup_{y\in \left[\frac{C_\ell}{C_u}, \frac{C_u}{C_\ell}\right]\setminus\{1\}}\frac{|y^\gamma-1|}{|y-1|}.
\end{equation}
It is easy to see that $\gammaf \in (0, 1).$ Finally, define the element-wise ratio
\[
    \p{u}{v}:=\diag(v)^{-1}u,
    \qquad u,v\in\bRp^{SA},
\]
and let $\ones \in \bRp^{SA}$ denote the vector of all ones.

Our first proposition establishes a key property of the relative error $\p{F(x)}{\xstr} - \ones$ that underpins our analysis.

\begin{proposition}[\textbf{One-point Contraction}]
\label{lem: local.con.F}
    For all $x\in\cK$,
    \[
        \left\|\p{F(x)}{\xstr}-\ones\right\|_\infty
        \leq
        \gammaf
        \left\|\p{x}{\xstr}-\ones\right\|_\infty,
    \]
    where $\gammaf$ is defined in~\eqref{e: def.gammaf}.
\end{proposition}
%
% \begin{proof}
%     See Appendix~\ref{s:1TS.Sync.Proof}.
% \end{proof}

Since $F(\xstr)=\xstr$, Proposition~\ref{lem: local.con.F} gives a one-point or pseudo contraction of the relative error toward $\xstr.$

Our next result gives a finite-time guarantee for the synchronous one-timescale algorithm from \eqref{e:x.synchronous.update}. 

\begin{theorem}[\textbf{Finite-time Convergence: One-timescale (Synchronous Case)}]\label{thm: 1TS-sync}
    Let $(x_n)$ be generated according to \eqref{e:x.synchronous.update} with $x_0 \in \bRp^{SA}.$ Suppose $\lambda > 0$ is chosen so that
    \[
        c_\lambda^{(s)} := \left( 1 -\gammaf^2\frac{SA(1+\lambda)}{ (SA + \lambda) } \right) > 0.
    \]
    Then, for the stepsize  $\alpha_n=\frac{a}{n+1}$ with $a c_\lambda^{(s)} > 1,$ we have
    \[
        \bE\left\| \big( \p{x_N}{\xstr} \big) - \ones\right\|_\infty \leq \frac{C'_{x,s}}{\sqrt{N}}, \quad \text{for } N>0.
    \]
    The constant $C'_{x,s}$ is specified in Appendix \ref{appendix:constants}.
\end{theorem}
%
% \begin{proof}
%     See Appendix~\ref{s:1TS.Sync.Proof}.
% \end{proof}

\begin{remark}
    The stepsize choice in the result above is universal: $C_\ell$, $C_u$, $\gammaf$, and hence a valid $c_\lambda^{(s)}>0$, are computable from $\gamma$, $x_0$, and the reward bound $\Rm$.
\end{remark}

\begin{remark}
    The $\cO(1/\sqrt{N})$ in the result above matches the folklore optimal rate, up to constant factors.
\end{remark}

We now provide our finite time guarantees for the asynchronous  algorithm in \eqref{e: 1TS} under Markovian sampling.

Let $\mu$ be a fixed behavior policy, and ${(s_n,a_n)}_{n\geq0}$ the resulting state-action trajectory, including any episodic resets. Further, let $(\cF_n)$ be the filtration of $\sigma$-algebras given by
\[
    \cF_n := \sigma\left(x_0, s_0, a_0, \ldots, s_n, a_n\right), \qquad n \geq 0.
\]

\begin{assumption}[\textbf{Geometrically Ergodic Sampling}]
\label{a: ergodicity}
The process $\{(s_n,a_n)\}_{n\geq0}$ admits a unique stationary
distribution $\eta_\mu\in(0,1)^{SA}$. Moreover, there exist constants
$C>0$ and $\rho\in(0,1)$ such that, for all $n\geq\tau\geq0$,
\[
    \left\|
    \bP\!\left((s_n,a_n)=\cdot\,\middle|\,\cF_{n-\tau}\right)
    -
    \eta_\mu
    \right\|_{\TV}
    \leq
    C\rho^\tau.
\]
\end{assumption}

Since $\cS\times\cA$ is finite, Assumption~\ref{a: ergodicity} holds whenever
the state-action transition matrix induced by $\mu$ and the reset mechanism
is irreducible and aperiodic.

Next, let $\delta:=\min_{s,a} \eta_\mu(s,a)$ and $\nu:= 1- \delta(1-\gammaf),$ where $\gammaf$ is as defined in \eqref{e: def.gammaf}. Due to Assumption~\ref{a: ergodicity}, it follows that $\delta > 0$ and, hence, $\nu \in (0, 1).$ Finally, let
\begin{equation}\label{e: def.c.lambda}
    c_\lambda := \left( 1 -\nu^2\frac{SA(1+\lambda)}{ (SA + \lambda) } \right),
\end{equation}
and, with respect to the stepsize $(\alpha_n)$ in \eqref{e: 1TS}, define 
\begin{equation}\label{e: def.taun} 
    \taun := \min\left\{ \tau>0 : C\rho^\tau < \alpha_n^2 \right\}, \qquad n \geq 0,
\end{equation}
and $\nstr:= \max\big\{ \nstr^{(1)}, \nstr^{(2)},\nstr^{(3)} \big\},$ where
\begin{equation}\label{e: def.nstar}
    \begin{aligned}
        \nstr^{(1)} &:= \min\left\{ n\geq 0: n>2\tau_n, \right\}
        \\
        \nstr^{(2)} & := \min\left\{ n\geq 0:  \alpha_n<\min(1/c_\lambda, \delta/6)\right\}
        \\
        \nstr^{(3)} &:= \min\left\{ n\geq 0: \alpha_{n-\taun} <2\alpha_n < 1/9 \right\},
    \end{aligned}
\end{equation}

The asynchronous analysis is more challenging because the mean drift of $x_n$ in \eqref{e: 1TS} is weighted by the stationary visitation distribution $\eta_\mu$. Specifically, define
\begin{equation}\label{e: bar.F}
    \bF(x):=(I-D_\mu)x+D_\mu F(x),
\end{equation}
where $D_\mu:=\diag(\eta_\mu)$ and $\eta_\mu$ is as defined in Assumption~\ref{a: ergodicity}.

Our first result here shows that the operator $\bF$ inherits the one-point contraction of $F$ from Lemma~\ref{lem: local.con.F}.
\begin{corollary}[\textbf{Mean Drift: One-point Contraction}]
\label{corol: local.con.bar.F}
    For all $x\in\cK$,  
    \[
          \left\|\p{\bF(x)}{\xstr}-\ones\right\|_\infty
        \leq
        \nu
        \left\|\p{x}{\xstr}-\ones\right\|_\infty.
    \]
\end{corollary}
Unlike Proposition~\ref{lem: local.con.F}, note that the contraction factor $\nu$ depends on the unknown $\eta_{\mu}.$

\begin{theorem}[\textbf{Finite-time Convergence: One-timescale (Asynchronous Case)}]\label{thm: 1TS}
    Suppose Assumption~\ref{a: ergodicity} holds and let $\lambda > 0$ be such that $c_\lambda > 0$. Let $(x_n)$ be generated according to \eqref{e: 1TS} with $x_0 \in \bRp^{SA}$. For the stepsize choice $\alpha_n= (n+1)^{-\alpha}$ with $\alpha\in(1/2,1),$ we have
    \[
        \bE\left\|\p{x_N}{\xstr} - \ones \right\|_\infty \leq C^{(\alpha)}_x\frac{\ln N}{N^{\alpha/2}}, \qquad N>\nstr.
    \]
    Further, for the stepsize $\alpha_n = \frac{a}{n+1},$ where $a c_\lambda>1,$ we have
    \[
        \bE\left\|\p{x_N}{\xstr} - \ones \right\|_\infty \leq C_x\frac{\ln N}{\sqrt{N}}, \qquad N>\nstr.
    \]
    See Appendix \ref{appendix:constants} for the definition of $C^{(\alpha)}_x$ and $C_x.$ 
\end{theorem}

Clearly, the optimal convergence rate of $\widetilde{O}(1/\sqrt{N})$ needs  the stepsize $\alpha_n$ to be set using $c_\lambda$, which depends on the underlying transition dynamics through $\nu$. In a typical RL setting, $\nu$ is unknown. A standard trick to overcome such a problematic dependence is iterate averaging, proposed independently by \cite{polyak1992acceleration} and \cite{ruppert1991stochastic}. In this scheme, a larger, albeit universal, stepsize is used in conjunction with iterate averaging. The next theorem gives the resulting parameter-free guarantee.
\begin{theorem}[\textbf{Parameter-Free Optimal Convergence}]\label{thm: 1TS.PR}
    Suppose Assumption~\ref{a: ergodicity} holds. Additionally, let the greedy policy w.r.t $\xstr,$ defined as $\pi_{\xstr}(s):=\argmin_a \xstr(s,a),$ assign a unique action to every state $s.$
    Let $(x_n)$ be generated according to \eqref{e: 1TS} with $x_0 \in \bRp^{SA}$ and $\alpha_n =(n+1)^{-\alpha},$ where $\alpha\in (1/2,1).$ Then, the running average $\bar{x}_N:=\frac{1}{N}\sum_{n=0}^{N-1}x_n$ of $x_0,\ldots,x_{N-1}$ satisfies 
    \[
        \bE\left\|\p{\bar{x}_N}{\xstr} - \ones \right\|_\infty \leq \bar{C}\frac{\ln N}{\sqrt{N}}, \qquad N>\nstr,
    \]
    where $\bar{C}>0$ is defined in Appendix \ref{appendix:constants}.
\end{theorem}
\begin{remark}[\textbf{Optimal Rate with a Universal Stepsize}]
    The Polyak--Ruppert averages $(\bar{x}_N)$ achieve the optimal rate
    $\widetilde{\cO}(1/\sqrt{N})$ using a stepsize that is independent of
    unknown model parameters.
\end{remark}

% \begin{remark}[\textbf{Synchronous Version}]
%     %
%     Analogous results hold for the synchronous updates given by~\eqref{e: 1TS.sync}. In particular, for $\alpha_n=\frac{1}{c(n+1)},$ for some problem-dependent $c$ $(x_n)$ generated by~\eqref{e: 1TS.sync} achieve a convergence rate of $\cO(1/\sqrt{N}).$ These rates can be achieved with parameter-free stepsizes via Polyak-Ruppert averaging.
%     %
% \end{remark}

\subsection{Convergence rates: Two-timescale}
For the two-timescale algorithm, although the Bellman operator \(T\) is contractive in the sup norm, this contraction can be exploited only after controlling how accurately the fast recursion makes $-\frac{\gamma}{\theta}\ln g_n$ track the moving target $TQ_n$. The resulting tracking error perturbs the slow recursion and must be controlled together with the mixing and contraction errors. 

We now give our key result for the two-timescale method.
\begin{theorem}[\textbf{Finite-time Convergence: Two-timescale}]\label{thm: 2TS}
    Suppose Assumption~\ref{a: ergodicity} holds. Let $(Q_n,g_n)$ be updated by~\eqref{e: 2TS}, with the stepsize choices $\beta_n=\frac{b}{(n+1)},$ such that $b(1-\gamma)>1,$ and $\alpha_n=\frac{\ln(n+1)}{n+1}.$ Then,  $    \bE\big\|
    g_N-e^{-\frac{\theta}{\gamma}TQ_N}
    \big\|_\infty
    \leq
    C_g\frac{\ln^{3/2} N}{\sqrt N},$ and $
    \bE\|Q_N-\Qstr\|_\infty
    \leq C_Q\frac{\ln^{3/2} N}{\sqrt N},$
    for all $N>\max(n^*,e^{1/\delta}).$ The constants $C_g$ and $C_Q$ are specified in Appendix \ref{appendix:constants}.
\end{theorem}

\begin{remark}[\textbf{Parameter-Free Optimal Rate}]
    The two-timescale iterates achieve the optimal $\widetilde{\cO}(1/\sqrt{N})$ rate without using parameter-dependent stepsizes. 
\end{remark}

\begin{remark}
    Unlike the one-timescale method, where parameter-free optimal rates are
    obtained through Polyak--Ruppert averaging, the two-timescale method
    achieves them directly with non-averaged iterates. Extending averaging to
    this setting is more delicate because the fast recursion tracks a moving
    target determined by $Q_n$.
\end{remark}

\section{Proof Outlines}
\label{s: proof.sketch}
Because $T$ is contractive in the standard $\|\cdot\|_\infty$ norm, the two-timescale analysis follows largely standard arguments, apart from the additional control of the tracking error. We therefore defer those proofs to Appendix~\ref{appendix: 2.TS}. Here, we sketch the one-timescale analysis, which departs substantially from standard techniques. We directly discuss the more challenging asynchronous case.

\subsection{Proof of Theorem~\ref{thm: 1TS}}
\label{s: 1TS.rate}
As highlighted in Section~\ref{s:main.results}, the additive one-timescale update in \eqref{e: 1TS} is not aligned with the contraction geometry of $F$ under the Thompson metric~\eqref{e: log.sup.metric}. This mismatch prevents a direct application of standard finite-time analyses for stochastic fixed-point iterations. The key ideas underlying our novel analysis can be summarized as follows. 

First, it is easy to see from $C_\ell$ and $C_u$'s definition given in \eqref{e: Cl.Cu.defn} that $\cK:=[C_\ell,C_u]^{SA}$ is a compact subset of $\bRp^{SA}.$ Next, a simple inductive argument shows that the iterates $(x_n)$ remain in $\cK$. In the scalar setting, i.e., when $S=A=1$, it is not difficult to see that $F(x) = (x/\xstr)^\gamma \xstr,$ i.e., it has a power-law structure. \cite[Theorem C.5]{thoppe2026reinforcement} combine the iterate  boundedness with the power-law structure of the $F$-operator to show that, in this \emph{scalar} case, the ratio $x_n/\xstr$ converges to $1$ at the optimal $\widetilde{O}(N^{-1/2})$ rate. Motivated by this observation, we focus the relative-error sequence $\left(x_n\oslash\xstr-\ones\right).$ 

Define the relative iterate $r_n:=\p{x_n}{\xstr}$. Multiplying \eqref{e: 1TS} by $\diag(\xstr)^{-1}$ then yields
\begin{multline}\label{e: 1TS.ratio}
    r_{n+1}-\ones =  r_n-\ones+\alpha_n\big[\p{\bF(x_n)}{\xstr} - \ones -(r_n - \ones)\big]
    \\
     +\alpha_n \big[\tF(x_n, s_n,a_n,s_{n+1})\big],
\end{multline}
where
\begin{multline}\label{e: noise}
    \tF(x,s,a,s') :=
    \left(D_\mu-e_{s,a}e_{s,a}^{\top}\right) \left(\p{x}{\xstr}\right)
    \\
    +
    \widehat{F}(x,s,a,s')
    \left(\p{e_{s,a}}{\xstr}\right)
    -
    D_\mu\left(\p{F(x)}{\xstr}\right).
\end{multline}
Although relative-error analysis is already nonstandard in the RL literature, the main difficulty is that the transformed drift
$\p{\bF(x_n)}{\xstr}-\ones$
has no immediate contraction property. Our key insight is to use the one-point contraction from Corollary~\ref{corol: local.con.bar.F}, which allows us to bring \eqref{e: 1TS.ratio} within the scope of Lyapunov-based stochastic fixed-point analysis; see \cite[Equation~(4)]{chen2021lyapunov}. In particular, we use the corollary to control
\[
    \left\|\p{\bF(x_n)}{\xstr} - \ones -(r_n - \ones)\right\|_\infty,
\]
while the temporally dependent noise $\tF(x_n,s_n,a_n,s_{n+1})$ is handled using standard Markov-noise arguments.

We now briefly describe how the remainder of our analysis proceeds. Since $\|\cdot\|_\infty$ is non-smooth, we employ the Moreau-envelope construction, first applied in an RL context by~\cite{chen2021lyapunov}. Specifically, we let $\lambda > 0$ and consider the smooth Lyapunov function
\begin{align}\label{e: lyapunov.function}
    \VL(r) := \min_{y} \left\{ \frac{1}{2}\|y\|_\infty^2 + \frac{1}{2\lambda}\|r-y\|_2^2 \right\}.
\end{align}
\citep[Lemma 2.1]{chen2020finite} shows that $\VL$ is convex, and satisfies $L_\lambda$-smoothness with respect to $\|\cdot\|_\infty,$ where $L_\lambda:= (\frac{SA}{\lambda}),$  i.e.,
\begin{multline*}
    \VL(r') \leq \VL(r) 
    + \left\langle \nabla\VL(r), r'-r \right\rangle 
    + \Big(\frac{L_\lambda}{2}\Big)\|r'-r\|^2_\infty.
\end{multline*}
Moreover,
\begin{equation}\label{e: lyapunov.norm}
            \frac{1}{(1+\lambda)}\ \|r\|^2_\infty \leq 2\VL(r) \leq \frac{SA}{(SA+\lambda)}\ \|r\|^2_\infty.
\end{equation}
Applying the smoothness inequality to~\eqref{e: 1TS.ratio} gives
\begin{align}\label{e: V.decomp}
    &\bE\VL(r_{n+1}-\ones) \leq \bE\VL(r_n-\ones)  
    \nonumber\\
    & \qquad\quad  + \alpha_n\bE\big\langle \nabla \VL(r_n-\ones) , \left(\p{\bF(x_n)}{\xstr}\right)-r_n\big\rangle 
    \nonumber\\
    & \qquad\quad + \alpha_n\bE\big\langle \nabla \VL(r_n-\ones) , \tF(x_n,s_n,a_n,s_{n+1})\big\rangle 
    \nonumber\\
    & \qquad\quad + \Big(\frac{L_\lambda}{2}\Big)\bE\|r_{n+1}-r_n\|^2_\infty.
\end{align}
The boundedness of $(x_n)$ implies that the last term is of order $\cO(\alpha_n^2)$. To control the Markovian-noise term, we 
add and subtract
\begin{equation}\label{e: inner.prod}
    \alpha_n\bE
    \big\langle
        \nabla\VL(r_{n-\taun}-\ones),
        \tF(x_{n-\taun},s_n,a_n,s_{n+1})
    \big\rangle.
\end{equation}
The resulting residual terms are controlled using the smoothness of
$\VL$, the Lipschitz continuity of  $\widehat{F}$ on $\cK$, and Young's inequality. In particular, the boundedness implies that the total contribution of the residual terms is bounded by
$\cO(\taun^2\alpha_{n-\taun}^2)$. The term
in~\eqref{e: inner.prod} is controlled using
Assumption~\ref{a: ergodicity}. Finally, using convexity of $\VL,$ we bound the second term in~\eqref{e: V.decomp} by 
\[
    \alpha_n\VL(( \p{\bF(x_n)}{\xstr}) - \ones) - \alpha_n\VL(r_n - \ones).
\]
We then use~\eqref{e: lyapunov.norm} to translate the local contraction property in Corollary~\ref{corol: local.con.bar.F} to a contraction in $\VL.$ This gives us the following result.
\begin{lemma}\label{lem: recurision.lyapunov}
    For all $n\geq \nstr,$ we have 
    \begin{multline*}
        \bE\VL(r_{n+1} - \ones) 
        \leq (1- c_\lambda\alpha_n) \bE\VL(r_n-\ones) + C_1\taun^2\alpha^2_n,
    \end{multline*}
    where $c_\lambda>0$ is as defined in Theorem~\ref{thm:  1TS} and $C_1$ is defined in the appendix.
\end{lemma}
Since $\taun=\cO(\ln n),$ the above lemma reduces the analysis to a standard one-step contractive recursion in which the additive term decays faster than the contractive term.~\cite[see][Equation 7]{chen2021lyapunov}. The desired convergence rates in terms of $\VL$ then follow from the argument in~\cite[Theorem~A.3]{chen2021lyapunov} for different stepsize choices, which then translate to bounds in $\|\cdot\|_\infty$ via~\eqref{e: lyapunov.norm}. \hfil\qed

\subsection{Proof of Theorem~\ref{thm: 1TS.PR}}
\label{s: proof.1TS.PR}

Polyak–Ruppert averaging is well understood for stochastic linear recursions~\cite{durmus2025finite}. It was subsequently extended to synchronous risk-neutral Q-learning by~\cite{li2023statistical} and to general asynchronous nonlinear fixed-point iterations by~\cite{naskar2026parameter}. Following the approach of~\cite{naskar2026parameter}, we begin by expressing the one-timescale update~\eqref{e: 1TS} as a perturbed linear recursion. 

Since the map $u\mapsto u^\gamma$ is increasing on $\bRp$, we can write
\[
    \bF(x)=(I-D_\mu)x + D_\mu A^x x^\gamma,
\]
where we let $\px$ denote the greedy policy with respect to $x\in\bRp^{SA},$ and define
\[
    A^{x}(s,a,s',a') 
    := e^{-\frac{\theta}{\gamma} r(s,a)}\ \cP(s'|s,a)\ \ones_{\{\px(s')=a'\}}.
\]
Thus, the
nonlinearity of $\bF$ arises from both the power-law map $x\mapsto x^\gamma$ and the dependence on the greedy policy $\pi_x.$ To isolate these two sources of nonlinearity, we first replace $A^{\pi_x}$ by $A^{\pi_{\xstr}}$ and then take the first-order Taylor approximation of $x^\gamma$ around $\xstr$.
This gives us the following decomposition:
\begin{align}\label{e: barF.linear.decomp}
    \bF(x) = \bFlin(x) + D_\mu R(x)+ D_\mu(A^{x} - A^{\xstr})x^\gamma,
\end{align}
where $R(x)$ arises from the second-order Taylor remainder associated with the power-law map, whereas the third term is the nonlinear perturbation induced by changes in the greedy
policy. Importantly, $\bFlin$ is admits the following affine form
\begin{align*}
        \left( \p{\bFlin(x)}{\xstr} - \ones \right) = B^*\left(\p{x}{\xstr} -\ones \right),
\end{align*}
where $B^*$ is a matrix satisfying $\|B^*\|_\infty \leq \big( 1- \delta( 1 -\gamma) \big).$

Using~\eqref{e: barF.linear.decomp}, we rewrite~\eqref{e: 1TS} as the following perturbed linear recursion:
\begin{align}\label{e: recursive.rel.err}
    r_{n+1}- \ones 
    & = \big[ I - \alpha_n (I - B^* )\big] \left( r_n - \ones \right)
    \nonumber\\
    & + \alpha_n\tF(x_n,s_n,a_n,s_{n+1}) + \alpha_n \xi_n,
\end{align}
where $\xi_n$ collects the nonlinear components in~\eqref{e: barF.linear.decomp}. We apply~\eqref{e: recursive.rel.err} recursively to obtain a closed form for $\Delta_n:=(r_n-\ones),$ followed by taking the average of $\Delta_0,\ldots,\Delta_{N-1}.$ This gives us the following decomposition of the averaged relative error $\bDelta_N:=\left( (\p{\bar{x}_n}{\xstr}) -\ones\right),$ for $N>\nstr:$
\begin{align*}
    \bDelta_N = \bDeltainit_N + \bDeltatrans_N + \bDeltalin_N + \bDeltanlin_N,
\end{align*}
where $G_{n_1:n_2} := \prod_{n=n_1}^{n_2-1} \left[ I - \alpha_n \big( I - B^* \big) \right],$ and 
\begin{align*}
    \bDeltainit_N & := \frac{1}{N}\sum_{n=0}^{n_*-1} \Delta_n, \\
    \bDeltatrans_N & := \frac{1}{N}\sum_{n=n_*}^{N-1} G_{n_*:n} \Delta_{n_*}
    \\
    \bDeltalin_N & := \frac{1}{N}\sum_{n=n_*}^{N-1}\sum_{k=n_*}^{n-1}\alpha_n G_{k+1:n}\tF(x_k,s_k,a_k,s_{k+1}),
    \\
    \bDeltanlin_N & := \frac{1}{N}\sum_{n=n_*}^{N-1}\sum_{k=n_*}^{n-1}\alpha_n G_{k+1:n}\xi_k.
\end{align*}

It now suffices to bound each of the above terms. We control the first term using the boundedness of the iterates $(x_n)$ as follows:
\begin{lemma}\label{lem: PR.init}
    For all $N>\nstr,$ $\bE\|\bDeltainit_N\|_\infty\leq \frac{\nstr}{N}\left(\frac{C_u}{C_\ell}\right).$
\end{lemma}

Next, since $\|I-\alpha_n(I-B^*)\|_\infty\leq \left(1-\alpha_n\delta(1-\gamma)\right),$ the iterated matrix product $\|G_{\nstr:n}\|$ admits an exponentially decaying upper bound. Taking sum over $n$ then gives us the following bound:
\begin{lemma}\label{lem: PR.trans}
    For all $N>\nstr,$ $\bE\|\bDeltatrans_N\|_\infty\leq \frac{ C^\textnormal{trans}}{N}.$
\end{lemma}
%
% \begin{equation}
%     %
%     \bE\|\bDeltatrans_N\|_\infty\leq \frac{ C^\textnormal{trans}}{N}.
%     %
% \end{equation}
%
To handle the remaining terms, we switch the order of the double summation and define $\upsilon_{k,N}:=\alpha_n\sum_{n=k+1}^{N-1} G_{k+1:n}.$ This gives us
\begin{align}
    \bDeltalin_N & := \frac{1}{N}\sum_{k=n_*}^{N-2}\upsilon_{k,N} \tF(x_k,s_k,a_k,s_{k+1}),\label{e: PR.noise}
    \\
    \bDeltanlin_N & := \frac{1}{N}\sum_{k=n_*}^{N-2}\upsilon_{k,N} \xi_k, \label{e: PR.nlin}
\end{align}
Using the contractive nature of $G_{n_1:n_2},$  we can show $\|\upsilon_{k,N}\|_\infty$ is uniformly bounded. Then, squaring both sides of~\eqref{e: PR.noise} and using Assumption~\ref{a: ergodicity} to handle inner products between noise terms, we get the following bound
\begin{lemma}\label{lem: PR.noise}
    For all $N>\nstr,$ $\bE\|\bDeltalin_N\|_\infty \leq \frac{\tau_N C^\textnormal{noise}}{\sqrt N}.$
\end{lemma}
%
% \[
%     \bE\|\bDeltalin_N\|_\infty \leq \frac{\tau_N^2 C^\textnormal{noise}}{\sqrt N}.
% \]
%
Finally, we bound the term arising from the nonlinear perturbation term $\xi_n.$
Since $R(x)$ arises from the Taylor remainder associated with $x^\gamma,$ we have $\|R(n)\|\leq \left\| x_n-\xstr\right\|^2= \cO(\left\| r_n - \ones \right\|^2).$ Further, since $x\mapsto \xp$ is piece-wise constant on $\bRp^{SA}$ there exists some $c_*,$ such that $A^x=A^{\xstr}$ if and only if $\|x-\xstr\|_\infty<c_*.$ This gives us
\begin{multline}
    \big\| A^x - A^{\xstr} \big\|_\infty  
        \leq 2\cdot\ones_{\{\|x-\xstr\|_\infty > c_* \}}
        \\
        \leq (2/c^2_*)\ \| x- \xstr \|^2_\infty = \cO\left(\|r_n-\ones\|^2_\infty\right).
\end{multline}
Hence, Theorem~\ref{thm: 1TS} implies $\bE\|\xi_n\|$ $=\cO(\bE\|r_n-\ones\|^2)=\cO(\taun^2\alpha_n).$ Finally, we apply triangle inequality on~\eqref{e: PR.nlin} to obtain the following result:
\begin{lemma}\label{lem: PR.nlin}
    For all $N>\nstr,$ $\bE\|\bDeltanlin_N\|_\infty \leq \frac{\tau_N^2 C^\textnormal{nlin}}{N^{\alpha}}.$
\end{lemma}
%
% \begin{equation}\label{e: PR.nlin}
%     %
%     \bE\|\bDeltanlin_N\|_\infty \leq \frac{\tau_N^2 C^\textnormal{nlin}}{N^{\alpha}}.
%     %
% \end{equation}
%
Finally, we add the bounds in Lemmas~\ref{lem: PR.init},~\ref{lem: PR.trans},~\ref{lem: PR.noise}, and,~\ref{lem: PR.nlin}. Using $\tau_N\leq \max(\ln C,4\alpha \ln N)/\ln(1/\rho),$ and taking maximum over the constants finishes the proof. \hfil\qed

\section{Conclusions and Future Work}
\label{sec:conclusions}
We performed a finite-time analysis of two recently proposed algorithms for discounted exponential-utility RL algorithms that operate using one and two timescales, respectively. Our bounds are optimal up to logarithmic factors for both algorithms, and are derived for stepsizes that are universal. The proofs of the finite-time bounds for the one-timescale algorithm depart significantly from prior works, and the proof technique involving local contractivity for these bounds, may be of independent interest. 

As future work, it would be interesting to incorporate feature-based representations and function approximation into the two aforementioned algorithms and establish finite-time guarantees for the resulting schemes.

\bibliography{references}

\newpage

\appendix

\section{Constants in the finite-time bounds}
\label{appendix:constants}

\begin{table}[ht]
\centering
\small
\setlength{\tabcolsep}{3pt}
\begin{tabular}{ll}
\hline
\textbf{Constant} & \textbf{Expression} \\
\hline
$C_1^{(s)}$
& $\dfrac{2L_\lambda C_u^2}{C_\ell^2}$ \\

$c_\lambda^{(s)}$
& $\gammaf^2u_V^2/\ell_V^2$ \\

$\ell_V$
& $\left(1+\frac{\lambda}{SA}\right)^{1/2}$ \\

$u_V$ & $\left(1+\lambda\right)^{1/2}$
\\

$c_V$&$\displaystyle
\frac{4u_V^2C_u^2}{\ell_V^2C_\ell^2}$\\

$C'_{x,s}$
& $\displaystyle
\left[c_V
(n_*+1)^{ac_\lambda^{(s)}}
+
2u_V^2C_1^{(s)}
C_{\Gamma,1}(c_\lambda^{(s)},a)
\right]^{1/2}
$ \\

$C_\Gamma(u,v)$
& $\dfrac{v^2 2^{vu}}{vu-1}$ \\

$C_1$
& $\displaystyle
(68 SA+80\lambda)
\left(\frac{SAC_u}{\lambda C_\ell}\right)^2
$ \\

$T_{n_*}$ & $\frac{\ln(n_* + 1)\ln(\frac{C}{\rho a^2 n_*^2}) }{\ln^2(n_*)\ln(1/\rho)}$
\\

$C_x$
& $\displaystyle
\left[
\left(2n_*c_V
+
u_V^2C_1C_\Gamma(c_\lambda,a)
\right)
T_{n_*}^2
\right]^{1/2}
$ \\

$C_{\Gamma,1}(u,v)$
& $\displaystyle
\frac{2}{u}
\max\!\left\{
2^v,\;
\left[
\frac{4e^v}{u(1-v)}
\right]^{\frac{v}{1-v}}
\right\}
$ \\

$C_{\Gamma,n_*}(u,v)$
& $\displaystyle
e^{\frac{u}{1-v}(n_*+1)^{1-v}}
$ \\

$C_j$ & $\displaystyle
\frac{2C_{\Gamma,n_*}(c_\lambda,\alpha)C_u^2}
{\ell_V^2C_\ell^2}$\\

$C_\Gamma'$
& $\displaystyle C_j
\left(
\frac{\alpha}{c_\lambda e}
\right)^{\frac{\alpha}{1-\alpha}}
+
C_1C_{\Gamma,1}(c_\lambda,\alpha)
$ \\

$C_r'$
& $\displaystyle
C_j
\left[
\frac{2e\alpha e^{ac_\lambda(k+1)^{1-\alpha}}}
{c_\lambda(1-\alpha)}
\right]^{-\frac{\alpha}{1-\alpha}}
+
C_1C_{\Gamma,1}(c_\lambda,\alpha)
$ \\

$C_x^{(\alpha)}$
& $\displaystyle
\left[
2u_V^2C_\Gamma'\left(\frac{\ln(N)\ln(n_* + 1)[\ln (C/\rho) + \alpha]}{\ln^2(n_*)\ln(1/\rho)}\right)^2
\right]^\frac{1}{2}$ \\

$C_{\mathrm{quad}}$
& $\displaystyle
\max\left\{
\frac{2C_u^2}{c_*^2},
e^{\frac{\theta}{\gamma}R_{\max}}
\frac{\gamma(1-\gamma)}{2}
\frac{C_u^2}{C_\ell^{2-\gamma}}
\right\}
$ \\

$C^{\mathrm{init}}$ & $\frac{2 C_u \nstr}{C_\ell}$ 
\\
$C^{\mathrm{noise}}$
& $2\left(\frac{6SAK_GC_u}{ C_\ell}\right)^2 (C_\alpha+1)^2$ \\

$C^{\mathrm{nonlin}}$
& $\displaystyle
\frac{
2u_V^2C_\Gamma'
C_{\mathrm{quad}}
(1+C_u^\gamma)}
{C_\ell\delta(1-\gamma)(1-\alpha)}
$ \\

$C^{\mathrm{trans}}$
& $\left[\frac{2}{e\delta(1-\gamma)}\right]^{\frac{\alpha}{1-\alpha}} \left(\frac{\pi^2 C_u}{3C_\ell}\right)e^{\frac{2\nstr\delta(1-\gamma)}{1-\alpha}}
$ \\

$\overline{C}$
& $\displaystyle
\max\left\{
n_*\frac{C_u}{C_\ell}+C^{\mathrm{trans}},
C^{\mathrm{noise}},
\frac{16\alpha^2C^{\mathrm{nonlin}}}
{\ln^2(1/\rho)}
\right\}
$ \\
\hline
\end{tabular}
\caption{Constants in the bounds for the one-timescale algorithm.}
\label{tab:constants-1TS}
\end{table}

\begin{table}[ht]
\centering
\small
\setlength{\tabcolsep}{3pt}
\begin{tabular}{ll}
\hline
\textbf{Constant} & \textbf{Expression} \\
\hline

$\Cm$
&
$(1+b+\ln\nstr)\frac{\Rm}{1-\gamma}
$
\\

$c_\theta$ & $e^{\frac{2\theta}{\gamma}\Cm}$\\

$C_{2,1}$
&
$\displaystyle
(SA)^2
\left[
41c_\theta
\tau_n^2\alpha_n^2
+
\theta^2C_{\mathrm{max}}^2
\left(
e^{2\theta}
+
2c_\theta
\right)
\right]
$
\\

$C_2$
&
$\displaystyle
C_{2,1}
+
(4c_\theta SA    
(8+3C_{\mathrm{max}}^2)
$
\\

$C_y'$
&
$\displaystyle
2c_\theta
(1+\theta^2C_{\mathrm{max}}^2)
(n_*+1)
+
\frac{8C_2}{\delta}
$
\\

$C_g$
&
$\displaystyle
\frac{2\ln C+4}{\ln(1/\rho)}
(C_y')^{1/2}
$
\\

$C_3$
&
$\displaystyle
\frac{\gamma\sqrt{SA}}{\theta}
e^{\frac{\theta}{\gamma}C_{\mathrm{max}}}
C_g
$
\\

$C_{\Gamma}(1-\gamma,\beta)$
&
$\displaystyle
\frac{b^2\,2^{b(1-\gamma)}}
{b(1-\gamma)-1}
$
\\

$C_Q$
&
$\displaystyle2C_{\mathrm{max}}(2n_*)^{b(1-\gamma)}+C_3C_{\Gamma}(1-\gamma,\beta)$
\\
\hline
\end{tabular}
\caption{Constants in the bounds for the two-timescale algorithm.}
\label{tab:constants-2TS}
\end{table}

\section{Proof of Theorem~\ref{thm: 1TS-sync}}
\label{s:1TS.Sync.Proof}

In this section, we derive finite-time convergence rates for the synchronous one-timescale algorithm defined in~\eqref{e:x.synchronous.update}. Throughout this appendix, for all $n\geq 0$, we use $\bE_n[\cdot]$ to denote the conditional expectation $\bE[\cdot\mid\cF_n].$

Under the synchronous update rule~\eqref{e:x.synchronous.update}, the relative error $r_n:=\big(\p{x_n}{\xstr} - \ones \big)$ satisfies
\begin{align}\label{e: ratio.scyn}
    \big(r_{n+1}-\ones\big) & = \big(r_n-\ones\big) +\alpha_n\left[\p{F(x_n)}{\xstr}  -r_n \right] 
    \nonumber\\
    & \quad +\alpha_n\omega_{n+1},
\end{align}
where
\[
    \omega_{n+1} := \p{\left(\hF_{n+1} - F(x_n)\right)}{\xstr}.
\]
We claim that~\eqref{e: ratio.scyn} is a contractive update with a rapidly decaying perturbation. To establish this claim, we first prove the local contraction property of $F$ near its fixed point $\xstr,$ as stated in Proposition~\ref{lem: local.con.F}.

\begin{proof}[\textbf{Proof of Proposition~\ref{lem: local.con.F}}]
    For all $x\in \cK,$ define
    \begin{equation}
        m_x:= \min_{s,a}r_x(s,a),\quad  M_x:= \max_{s,a}r_x(s,a),
    \end{equation}
    where $r_x:= \p{x}{\xstr}.$ Then, for $x\in\cK,$
    \begin{multline}\label{e: F.upper}
        F(x) = F(r_x\xstr) \overset{(a)}{\preceq} F(M_x\xstr)
        \\
        \overset{(b)}{=} M_x^\gamma F(\xstr) \overset{(c)}{=} M_x^\gamma \xstr,
    \end{multline}
    where $(a)$ follows since $F$ is monotone, $(b)$ follows since $F$ is homogeneous with degree $\gamma$, and $(c)$ follows since $\xstr$ is the unique fixed point of $F.$ Similar arguments give
    \begin{multline}\label{e: F.lower}
        F(x) = F(r_x\xstr) \succeq F(m_x\xstr)
        \\
        = m_x^\gamma F(\xstr) = m_x^\gamma \xstr.
    \end{multline}
    Multiplying~\eqref{e: F.upper} and~\eqref{e: F.lower} by $\diag(\xstr)^{-1}$ gives us the following sandwiching relation:
    \begin{equation}\label{e: F.sandwich}
        m_x^\gamma \leq \big(\p{F(x)}{\xstr}\big) \leq M_x^\gamma, \quad \text{for all} \quad x\in\cK.
    \end{equation}
    Subtracting $\ones$ from both sides and taking norms, we obtain
    \begin{align*}
        \left\|\p{F(x)}{\xstr} - \ones \right\|_\infty 
        & \leq \max\left\{ |M^\gamma_x-1| , |m_x^\gamma-1|\right\}
        \\
        & \leq \gammaf\cdot\max\left\{ |M_x-1| , |m_x-1|\right\}
        \\
        & = \gammaf  \left\|r_x - \ones \right\|_\infty,
    \end{align*}
    where
    \[
        \gammaf := \sup_{y\in \left[\frac{C_\ell}{C_u}, \frac{C_u}{C_\ell}\right]\setminus\{1\}}\frac{|y^\gamma-1|}{|y-1|}<1.
    \]
    This completes the proof of Proposition~\ref{lem: local.con.F}.
\end{proof}

Having established the local contraction property of $F$ in Proposition~\ref{lem: local.con.F}, we now introduce an appropriate Lyapunov function. Since the $\|\cdot\|_\infty$ is a  non-smooth norm, we fix a parameter $\lambda>0$ and define $\VL:\bR^{SA}\to\bR$ by
\begin{align}\label{e: def.lyapunov}
    \VL(r) := \min_{y} \left\{ \frac{1}{2}\|y\|_\infty^2 + \frac{1}{2\lambda}\|r-y\|_2^2 \right\}.
\end{align}
Then, we have the following result from~\cite[Lemma 2.1]{chen2020finite}:
\begin{lemma}\label{lem: ceecee}
    The function $\VL$ satisfies the following:
    \begin{enumerate}
        \item $\VL$ is convex, and is $L_\lambda$-smooth with respect to $\|\cdot\|_\infty,$ where $L_\lambda:= \frac{SA}{\lambda},$ i.e., for all $r,r',$
        \begin{align}\label{e: lyapunov.smooth}
            \VL(r') &\leq \VL(r) 
            \nonumber\\
            & + \left\langle \nabla\VL(r), r'-r \right\rangle 
            + \frac{L_\lambda}{2}\|r'-r\|^2_\infty,
        \end{align}
        \item For all $r,$ we have
        \begin{equation}\label{e: lyapunov.equivalent}
            (1/u^2_V)\ \|r\|^2_\infty \leq 2\VL(r) \leq (1/\ell^2_V)\ \|r\|^2_\infty,
        \end{equation}
        where $\ell_V:= \sqrt{1+\frac{\lambda}{SA}}$ and $u_V:= \sqrt{1+\lambda}.$
    \end{enumerate}
\end{lemma}
Combining~\eqref{e: lyapunov.smooth} with the update~\eqref{e: ratio.scyn}, we get for all $n\geq 0,$
\begin{align}\label{e: ratio.sync.decomp}
    & \bE\VL(r_{n+1}-\ones) 
    \nonumber\\
    & \leq \bE\VL(r_n-\ones)  + \frac{L_\lambda}{2}\bE\|r_{n+1}-r_n\|^2_\infty 
    \nonumber\\
    & \quad + \bE\left\langle \nabla \VL(r_n-\ones) , r_{n+1}-r_n\right\rangle
    \nonumber\\
    & \overset{(a)}{=} \bE\VL(r_n-\ones) + \alpha^2_n\frac{L_\lambda}{2}\bE\|\p{\hF_{n+1}}{\xstr} - r_n\|^2_\infty
    \nonumber\\
    & \quad+ \alpha_n\bE\left\langle \nabla \VL(r_n-\ones) , \big(\p{F(x_n)}{\xstr}\big) - r_n\right\rangle 
    \nonumber\\
    & \quad + \alpha_n\bE\left\langle \nabla \VL(r_n-\ones) , \bE_n[\omega_{n+1}]\right\rangle
    \nonumber\\
    & \overset{(b)}{=} \bE\VL(r_n-\ones) + \alpha^2_n\frac{2 L_\lambda C^2_u}{C^2_\ell}
    \nonumber\\
    & \quad+ \alpha_n\bE\left\langle \nabla \VL(r_n-\ones) , \big(\p{F(x_n)}{\xstr}\big) - r_n\right\rangle
    \nonumber\\
    & \overset{(c)}{\leq } \bE\VL(r_n-\ones) + \alpha^2_n\frac{2 L_\lambda C^2_u}{C^2_\ell}
    \nonumber\\
    & \quad+ \alpha_n\bE\VL\big(\p{F(x_n)}{\xstr}\big) - \ones\big) - \alpha_n\bE\VL\big( r_n - \ones\big),
\end{align}
where $(a)$ follows from the update~\eqref{e: ratio.scyn} and the fact that $r_n$ is $\cF_n$-measurable, $(b)$ follows since
\begin{equation}
    \cF_n \omega_n = \p{\big(\bE_n\hF_{n+1} - F(x_n)\big)}{\xstr} = 0,
\end{equation}
and since $\|\hF_{n+1}\|_\infty \leq C_u/C_\ell$ by~\cite[Lemma C.1]{thoppe2026reinforcement}. Finally, $(c)$ uses convexity of $\VL.$ 

Further, note that
\begin{align}
    \VL\left( \p{F(x_n)}{\xstr} -\ones \right) 
    & \overset{(a)}{\leq} \frac{1}{2\ell^2_V}\left\| \p{F(x_n)}{\xstr}-\ones \right\|^2_\infty 
    \nonumber\\
    &  \overset{(b)}{\leq} \gammaf^2\frac{1}{2\ell^2_V}\left\| r_n-\ones \right\|^2_\infty 
    \nonumber\\
    &  \overset{(c)}{=} \gammaf^2\frac{u^2_V}{\ell^2_V}\ \VL( r_n-\ones ),
\end{align}
where $(a)$ and $(c)$ use~\eqref{e: lyapunov.equivalent}, and $(b)$ follows from Proposition~\ref{lem: local.con.F}.
Plugging above inequality into~\eqref{e: ratio.sync.decomp}, we have for $n\geq \nstr,$
\begin{multline}
    \bE\VL(r_{n+1} - \ones) 
    \\
    \leq \big(1-c^{(s)}_{\lambda}\alpha_n \big)\bE\VL(r_n-\ones) + C^{(s)}_1\alpha^2_n.
\end{multline}
This yields the desired one-step recursion in the standard form of~\cite[Equation~(7)]{chen2021lyapunov}. Unrolling this recursion, for any $N>\nstr$, we obtain
\begin{align}
    &\bE\VL(r_N - \ones) 
    \nonumber\\
    &\leq  \Gamma^{(s)}_{\nstr:N}\VL(r_{\nstr}-\ones) + C^{(s)}_1\sum_{n=\nstr}^{N-1}\alpha^2_n\ \Gamma^{(s)}_{n+1:N}
    \nonumber\\
    &\leq  \frac{1}{2\ell^2_V}\ \Gamma^{(s)}_{\nstr:N}\ \|r_0-\ones\|_\infty^2 + C^{(s)}_1\sum_{n=\nstr}^{N-1}\alpha^2_n\ \Gamma^{(s)}_{n+1:N}
    \nonumber\\
    &\leq  \left(\frac{2C^2_u}{\ell^2_VC^2_\ell}\right)\ \Gamma^{(s)}_{0:N} + C^{(s)}_1\sum_{n=\nstr}^{N-1}\alpha^2_n\ \Gamma^{(s)}_{n+1:N},
\end{align}
where we define the iterated products
\[
    \Gamma^{(s)}_{n_1:n_2} = \prod_{n = n_1}^{n_2 -1} \left(1-c^{(s)}_\lambda\alpha_n \right), \quad \text{for all }\nstr\leq n_1<n_2.
\]
Applying Lemma~\ref{lem: Gnn.bound} under the condition $ac_\lambda^{(s)}>1,$ we obtain
\begin{align}
    \bE\VL(r_N - \ones) & \leq \left(\frac{2C^2_u}{\ell^2_VC^2_\ell}\right)\left(\frac{\nstr+1}{N+1}\right)^{ac^{(s)}_\lambda}
    \nonumber\\
    & \quad + \frac{C_{1}^{(s)}C_{\Gamma,1}(c^{(s)}_\lambda,a)}{N+1}.
\end{align}
Finally, using~\eqref{e: lyapunov.equivalent}, we get
\begin{align*}
    \bE\|r_N - \ones\|^2_\infty & \leq \left(\frac{4u^2_VC^2_u}{\ell^2_VC^2_\ell}\right)\left(\frac{\nstr+1}{N+1}\right)^{ac^{(s)}_\lambda} 
    \\
    & \quad + \frac{2u^2_VC^{(s)}_1C_{\Gamma,1}(c^{(s)}_\lambda,a)}{N+1}
    \\
    & \overset{(a)}{ \leq } \frac{\left(C'_{x,s}\right)^2}{N+1},
\end{align*}
where, in $(a)$, we use $(N+1)^{-ac_\lambda^{(s)}}\leq 1/(N+1)$, which follows from $ac_\lambda^{(s)}>1,$ together with the definition of $C'_{x,s}$ in Table~\ref{tab:constants-1TS}. 

Theorem~\ref{thm: 1TS-sync} then follows from Jensen's inequality. \hfill\qed

\section{Proof of Theorem~\ref{thm: 1TS}}
\label{s:1TS.async.Proof}

In this section, we prove Theorem~\ref{thm: 1TS} following the approach outlined in Section~\ref{s: 1TS.rate}.

To derive the standard contractive one-step recursion in~\cite[Equation~(7)]{chen2021lyapunov}, we first establish the local contraction properties of the modified mean operator $\bF,$ as stated in Corollary~\ref{corol: local.con.bar.F}.

\begin{proof}[\textbf{Proof of Corollary~\ref{corol: local.con.bar.F}}]

Recall from~\eqref{e: bar.F}, that the modified mean operator $\bF$ is defined as
\[
    \bF(x) := D_\mu F(x) + (I-D_\mu)x,
\]
where $D_\mu:=\diag(\eta_\mu).$ To obtain Corollary~\ref{corol: local.con.bar.F}, note that for all $x\in\cK,$ and state-action pair $(s,a),$ we have
\begin{align*}
    & \frac{\p{\big(\bF(x)}{\xstr}\big)(s,a)-1}{\|\p{x}{\xstr}-\ones\|_\infty} 
    \\
    & = \eta_\mu(s,a)\frac{\big(\p{\big(F(x)}{\xstr}\big)(s,a)-1\big)}{\|\p{x}{\xstr}-\ones\|_\infty}  
    \\
    & \qquad + (1-\eta_\mu(s,a))\frac{\big(\p{x}{\xstr}\big)(s,a)-1}{\|\p{x}{\xstr}-\ones\|_\infty} 
    \\
    & \overset{(a)}{\leq} \eta_\mu(s,a)\frac{\|\p{\big(F(x)}{\xstr}\big)(s,a)-\ones\|_\infty}{\|\p{x}{\xstr}-\ones\|_\infty}  
    \\
    & \qquad + (1-\eta_\mu(s,a))\frac{\|\p{x}{\xstr}-\ones\|_\infty}{\|\p{x}{\xstr}-\ones\|_\infty} 
    \\
    & \overset{(b)}{\leq}  \eta_\mu(s,a)\gammaf  + (1-\eta_\mu(s,a))
    \\
    & = 1 - \eta_\mu(s,a)(1-\gammaf) \leq 1 - \delta(1-\gammaf),
\end{align*}
where $(a)$ uses since for any vector $x,$ $x(s,a)\leq |x(s,a)|\leq \|x\|_\infty,$ while $(b)$ follows from Proposition~\ref{lem: local.con.F}.

The claim follows as we use a similar argument to obtain a lower bound of $\delta(1-\gammaf)-1,$ implying that
\[
    \frac{|\p{\big(\bF(x)}{\xstr}\big)(s,a)-1|}{\|\p{x}{\xstr}-\ones\|_\infty} \leq 1 - \delta(1-\gammaf).
\]
This completes the proof of Corollary~\ref{corol: local.con.bar.F}.
\end{proof}

Having established the local contraction property of $\bF$ in Corollary~\ref{corol: local.con.bar.F}, we next rewrite the one-timescale update~\eqref{e: 1TS} as a contractive stochastic recursion perturbed by a rapidly decaying Markovian noise term. To this end, observe that the drift term in~\eqref{e: 1TS} can be written as
\begin{align*}
    \Big[& \hF(x_n,s_n,a_n,s_{n+1}) - x(s_n,a_n) \Big] e_{s_n,a_n}
    \\
    & = \hF(x_n,s_n,a_n,s_{n+1})e_{s_n,a_n} - x_n
    \\
    & + (I - e_{s_n,a_n}e_{s_n,a_n}^\top)x_n 
    \\
    & = D_\mu F(x_n) + (I-D_\mu)x_n -x_n
    \\
    & + (D_\mu - e_{s_n,a_n}e^\top_{s_n,a_n})x_n
    \\
    & + \hF(x_n,s_n,a_n,s_{n+1})e_{s_n,a_n} - D_\mu F(x_n)
    \\
    & = \bF(x_n) - x_n + \diag(\xstr)\tF(x,s_n,a_n,s_{n+1})
\end{align*}
where, 
\begin{multline*}
    \tF(x,s,a,s') :=
    \left(D_\mu-e_{s,a}e_{s,a}^{\top}\right) \left(\p{x}{\xstr}\right)
    \\
    +
    \widehat{F}(x,s,a,s')
    \left(\p{e_{s,a}}{\xstr}\right)
    -
    D_\mu\left(\p{F(x)}{\xstr}\right).
\end{multline*}
Thus, multiplying~\eqref{e: 1TS} by $\diag(\xstr)^{-1}$ gives us
\begin{multline}\label{e: ratio}
    \left(r_{n+1}-\ones\right) = \left( r_n-\ones \right) +\alpha_n\big[\p{\bF(x_n)}{\xstr}-r_n\big]
    \\
     +\alpha_n \big[\tF(x_n, s_n,a_n,s_{n+1})\big],
\end{multline}
where $\tF(x_n, s_n,a_n,s_{n+1})$ denotes the Markovian noise associated with $\p{\bF(x_n}{\xstr}).$ 

Next, we employ the smooth Lyapunov function $\VL$ defined in~\eqref{e: def.lyapunov}. Using the properties of $\VL$ established in Lemma~\ref{lem: ceecee}, together with the local contraction property of $\bF$ from Corollary~\ref{corol: local.con.bar.F}, we prove Lemma~\ref{lem: recurision.lyapunov}, which shows that, for all $n\geq\nstr,$
\[
    \bE\VL(r_{n+1} - \ones) \leq (1-c_\lambda\alpha_n)\bE\VL(r_n-\ones) + \cO(\taun^2\alpha^2_n),
\]
where $c_\lambda$ and $\taun$ are defined in~\eqref{e: def.c.lambda} and immediately after~\eqref{e: def.nstar}, respectively.

\begin{proof}[\textbf{Proof of Lemma~\ref{lem: recurision.lyapunov}}]

Applying $\VL$ to $(r_{n+1}-\ones),$ and using~\eqref{e: ratio}, we have for all $n\geq \nstr,$
\begin{align}\label{e: error.rn.decomp}
    & \bE\VL(r_{n+1}-\ones) 
    \nonumber\\
    & \overset{(a)}{\leq} \bE\VL(r_n-\ones)  + \frac{L_\lambda}{2}\bE\|r_{n+1}-r_n\|^2_\infty 
    \nonumber\\
    & \qquad + \bE\left\langle \nabla \VL(r_n-\ones) , r_{n+1}-r_n\right\rangle
    \nonumber\\
    & \overset{(b)}{=} \bE\VL(r_n-\ones) + \frac{L_\lambda}{2}\bE\|r_{n+1}-r_n\|^2_\infty
    \nonumber\\
    & \qquad+ \alpha_n\bE\left\langle \nabla \VL(r_n-\ones) , \p{\bar{F}(x_n)}{\xstr} - r_n\right\rangle 
    \nonumber\\
    & \qquad + \alpha_n\bE\left\langle \nabla \VL(r_n-\ones) , \tF(x_n,s_n,a_n,s_{n+1})\right\rangle,
\end{align}
where $(a)$ follows from~\eqref{e: lyapunov.smooth}, $(b)$ uses the update rule~\eqref{e: ratio}.

We now bound each term in~\eqref{e: error.rn.decomp} one-by-one. Note that
\begin{align}\label{e: bound.one.step.r}
    & \| r_{n+1} -r_n\|^2_\infty 
    \nonumber\\
    & \overset{(a)}{=} \alpha^2_n \left( \hF(x_n,s_n,a_n,s_{n+1}) - x_n(s_n,a_n) \right)^2
    \nonumber\\
    & \quad \times \|\p{e_{s_n,a_n}}{\xstr}\|^2_\infty
    \nonumber\\
    & \overset{(b)}{\leq} \left(\hF(x_n,s_n,a_n,s_{n+1}) - x_n(s_n,a_n)\right)^2 \left(\frac{\alpha^2_n}{C^2_\ell}\right)
    \nonumber\\
    & \overset{(c)}{\leq} \left(\frac{4C^2_u}{C^2_\ell}\right) \alpha^2_n,
\end{align}
where $(a)$ follows from the update rule~\eqref{e: 1TS} and the definition of $r_n,$ $(b)$ follows since $\xstr\in \cK,$ and $(c)$ follows since $x_n\in \cK,$ and by~\cite[Lemma C.1]{thoppe2026reinforcement}, $\hF(x_n,s_n,a_n,s_{n+1})\in [C_\ell,C_u].$ Thus, we conclude
\begin{equation}
    \frac{L_\lambda}{2}\ \bE\|r_{n+1}-r_n\|^2_\infty \leq \left(\frac{2L_\lambda C^2_u}{C^2_\ell}\right) \alpha^2_n.
\end{equation}
Next, using the convexity of $\VL,$ we have
\begin{align}\label{e: rn.decomp.convex}
    \big\langle \nabla & \VL(r_n-\ones), \left(\p{\bar{F}(x_n)}{\xstr}\right) - r_n\big\rangle
    \nonumber\\
    & = \langle \nabla \VL(r_n-\ones) ,\left(\left(\p{\bar{F}(x_n)}{\xstr}\right) - \ones\right)  - \big(r_n-1)\big\rangle
    \nonumber\\
    & \leq   \VL\left(\left(\p{\bar{F}(x_n)}{\xstr}\right) - \ones\right) - \VL\big(r_n-1\big).
\end{align}
Moreover,
\begin{align}\label{e: rn.decomp.contract}
    \VL\left( \left(\p{\bF(x_n)}{\xstr}\right) - \ones \right) 
    & \overset{(a)}{\leq} \frac{1}{2\ell^2_V}\left\| \left(\p{\bF(x_n)}{\xstr}\right) -\ones \right\|^2_\infty 
    \nonumber\\
    &  \overset{(b)}{\leq} \frac{\nu^2}{2\ell^2_V}\left\| r_n-\ones \right\|^2_\infty 
    \nonumber\\
    &  \overset{(c)}{\leq} \nu^2\ \frac{u^2_V}{\ell^2_V}\ \VL( r_n-\ones ),
\end{align}
where $(a)$ and $(c)$ follow from~\eqref{e: lyapunov.equivalent}, while $(b)$ follows from Corollary~\ref{corol: local.con.bar.F}. Thus, combining~\eqref{e: rn.decomp.convex} and~\eqref{e: rn.decomp.contract}, we conclude that
\begin{multline}\label{e: r.contract}
    \left\langle \nabla \VL(r_n-\ones), \left(\p{\bar{F}(x_n)}{\xstr}\right) - r_n\right\rangle 
    \\
    \leq -\left( 1 -\nu^2\ \frac{u^2_V}{\ell^2_V} \right) \VL( r_n-\ones).
\end{multline}
Plugging the bounds in~\eqref{e: bound.one.step.r} and~\eqref{e: r.contract} into the recursive relation~\eqref{e: error.rn.decomp}, we have for $n\geq \nstr,$
\begin{align}\label{e: error.rn.recurse}
    & \bE\VL(r_{n+1}-\ones) 
    \nonumber\\
    & \leq \big( 1 -c_\lambda\alpha_n \big)\bE\VL(r_n-\ones) + \left(\frac{2L_\lambda C^2_u}{C^2_\ell}\right) \alpha^2_n
    \nonumber\\
    & \qquad + \alpha_n\bE\left\langle \nabla \VL(r_n-\ones) , \tF(x_n,s_n,a_n,s_{n+1})\right\rangle,
\end{align}
where $c_\lambda$ is as defined in~\eqref{e: def.c.lambda}. To bound the inner-product term
\[
    T_n:= \bE\left\langle \nabla \VL(r_n-\ones) , \tF(x_n,s_n,a_n,s_{n+1})\right\rangle, 
\]
we use a conditioning argument based on Assumption~\ref{a: ergodicity}. We construct the following decomposition:
\begin{equation}\label{e: T.decomp}
    T_n = T_{n,1} + T_{n,2} + T_{n,3},
\end{equation}
where, letting $v_n$ and $z_n$ denote $\nabla\VL(r_n-1)$ and  $(s_n,a_{n+1}, s_{n+1}),$ respectively, we define
\begin{equation}
    \begin{aligned}
    T_{n,1} & := \bE\left\langle v_{n-\tau_n}, \tF(x_{n-\tau_n}, z_n)  \right\rangle 
    \\
    T_{n,2} & := \bE\left\langle v_{n-\tau_n}, \tF(x_n, z_n) - \tF(x_{n-\tau_n}, z_n)\right\rangle,
    \\
    T_{n,3} & := \bE\left\langle v_n - v_{n-\tau_n}, \tF(x_n,z_n) \right\rangle.
    \end{aligned}
\end{equation}
Note that the above decomposition is valid since, for $n\geq \nstr\geq \nstr^{(1)},$
we have $n>2\taun.$ Note that, $T_{n,1}$ satisfies
\begin{align}
    & \bE_{n-\taun} T_{n,1}
    \nonumber\\
    & \overset{(a)}{=} \bE\left\langle v_{n-\tau_n}, \bE_{n-\taun}\tF(x_{n-\tau_n}, z_n)  \right\rangle
    \nonumber\\
    & \overset{(b)}{\leq} \frac{\alpha_n}{2} \left\| v_{n-\taun}\right\|^2_2 + \frac{1}{2\alpha_n}\big\|\bE_{n-\taun}\tF(x_{n-\tau_n}, z_n)\big\|^2_2
    \nonumber\\
    & \leq \frac{\alpha_n}{2} (SA)\|v_{n-\taun}\|^2_\infty + \frac{(SA)}{2\alpha_n}\big\|\bE_{n-\taun}\tF(x_{n-\tau_n}, z_n) \big\|^2_\infty
    \nonumber\\
    & \overset{(c)}{\leq} \frac{\alpha_n}{2} (SA)L_\lambda\|r_{n-\taun}-\ones\|^2_\infty + \frac{(2SA)}{\alpha_n}\left(\frac{36C^2_u}{C^2_\ell}\right)C^2\rho^{2\taun}
    \nonumber\\
    & \overset{(d)}{\leq} \frac{\alpha_n}{2} (SA)L_\lambda\|r_{n-\taun}-\ones\|^2_\infty + (2SA)\left(\frac{36C^2_u}{C^2_\ell}\right)\alpha^3_n 
    \nonumber\\
    & \overset{(e)}{\leq} \alpha_n (SA)L_\lambda\|r_n-\ones\|^2_\infty 
    + (2SA)\left(\frac{36C^2_u}{C^2_\ell}\right)\alpha^3_n 
    \nonumber\\
    & \quad + \alpha_n (SA)L_\lambda\|r_n - r_{n-\taun}\|^2_\infty,
\end{align}
where, $(a)$ follows since $r_{n-\taun}$ is $\cF_{n-\taun}$-measurable, $(b)$ uses Young's inequality, $(c)$ follows from Assumption~\ref{a: ergodicity} and the fact that $x\in \cK$ implies $\hF(x,s,a,s')\in [C_\ell, C_u],$ for all $s,a,s',$ $(c)$ follows since $\nabla \VL(0)=0$ and $\nabla\VL$ is $L_\lambda$-Lipschitz ($L_\lambda$-smoothness), and $(d)$ folows from the definition of $\taun.$ Finally, $(e)$ uses the inequality $(p+q)^2\leq 2(p^2+q^2).$ 

Hence, by the tower property of expectation, we have
\begin{align}\label{e: bound.T11}
    \alpha_n\bE T_{n,1} & = \alpha_n\bE\left[\bE_{n-\taun}T_{n,1}\right]
    \nonumber\\
    & \leq \alpha^2_n (SA)L_\lambda\|r_n-\ones\|^2_\infty 
    + (SA)\left(\frac{36C^2_u}{C^2_\ell}\right)\alpha^4_n 
    \nonumber\\
    & \qquad + \alpha^2_n (SA)L_\lambda\|r_n - r_{n-\taun}\|^2_\infty.
\end{align}
We apply Young's inequality to $T_{n,2}$ to obtain
\begin{align}\label{e: bound.T21}
    \alpha_n& \bE T_{n,2} 
    \nonumber\\
    & \leq \frac{\alpha^2_n}{2}\bE\left\| v_{n-\tau_n} \right\|^2_2 
    \nonumber\\
    & \quad + \frac{1}{2}\bE\big\|\tF(x_n, z_n) - \tF(x_{n-\tau_n}, z_n)\big\|^2_2
    \nonumber\\
    & \leq \frac{\alpha^2_n}{2}(SA)\bE\left\| v_{n-\tau_n} \right\|^2_\infty 
    \nonumber\\
    & \quad + \frac{(SA)}{2}\bE\big\|\tF(x_n, z_n) - \tF(x_{n-\tau_n}, z_n)\big\|^2_\infty
    \nonumber\\
    & \overset{(a)}{\leq} \frac{\alpha^2_n}{2}(SA)L_\lambda\bE\left\| r_{n-\tau_n} -\ones \right\|^2_\infty 
    \nonumber\\
    & \quad + (SA)(4+\nu^2)\bE\big\|r_n - r_{n-\tau_n}\big\|^2_\infty
    \nonumber\\
    & \overset{(b)}{\leq} \alpha^2_n(SA)L_\lambda\bE\left\| r_n -\ones \right\|^2_\infty 
    \nonumber\\
    &  \quad + (SA)\left(2L_\lambda + 5\right) \bE\big\|r_n - r_{n-\tau_n}\big\|^2_\infty
\end{align}
where $(a)$ combines the facts that $\nabla \VL(0)=0$ and $\VL$ is $L_\lambda$-smooth, and $\tF(x,z)$ is $(2+\nu)$-Lipschitz in $x,$ $(b)$ uses the inequality $(p+q)^2\leq 2(p^2+q^2)$ and the fact that $\nu^2<1.$

Similar arguments show that
\begin{align}\label{e: bound.T31}
    &\alpha_n \bE T_{n,3} 
    \nonumber\\
    & \overset{(a)}{\leq} (SA)\|v_n - v_{n-\taun}\|^2_\infty
    + \alpha^2_n(SA) \big\|\tF(x_n,z_n)\big\|^2_\infty
    \nonumber\\
    & \overset{(b)}{\leq} (SA)L_\lambda \|r_n - r_{n-\taun}\|^2_\infty + \alpha^2_n(36SA)\left(\frac{C^2_u}{C^2_\ell} \right),
\end{align}
where $(a)$ uses $\|\cdot\|_2\leq \sqrt{SA}\|\cdot\|_\infty$ and Young's inequality, whereas $(b)$ uses the fact that $\hF(x,s,a,s')e_{s,a}\in \cK$ for $x\in \cK$~\cite[see][Lemma C.1]{thoppe2026reinforcement}, implying that $\|(\hF(x,s,a,s') - x)e_{s,a}+x\|_\infty \leq  3C_u$ and $\|\bF(x)\|_\infty\leq 3C_u;$ consequently $\|\tF(x,s,a,s')\|_\infty\leq 6C_u/C_\ell$ since $\xstr\in \cK.$

Now, to bound $\| r_n - r_{n-\tau_n}\|^2_\infty,$ note that
\begin{align}\label{e: bound.mix.diff}
    & \|r_n - r_{n-\tau_n}\|^2_\infty \leq \left\|\sum_{k=n-\tau_n}^{n-1} (r_{k+1} - r_k)\right\|^2_\infty
    \nonumber\\
    & \overset{(a)}{\leq} \tau_n\sum_{k=n-\tau_n}^{n-1}\left\| r_{k+1} - r_k\right\|^2_\infty
    \nonumber\\
    & \overset{(b)}{\leq} \tau_n\sum_{k=n-\tau_n}^{n-1}\alpha^2_k\left(\frac{4C^2_u}{C^2_\ell}\right)
    \nonumber\\
    & \overset{(c)}{\leq} \tau^2_n\alpha^2_{n-\taun}\left(\frac{4C^2_u}{C^2_\ell}\right)
    \overset{(d)}{\leq} \tau^2_n\alpha^2_n\left(\frac{16C^2_u}{C^2_\ell}\right),
\end{align}
where $(a)$ uses the inequality $\big(\sum_{k=1}^m r_k\big)^2\leq m\sum_{k=1}^k r^2_k,$ $(b)$ follows from~\eqref{e: bound.one.step.r}, $(c)$ follows since $(a_n)$ is non-increasing, and $(d)$ follows since $2\alpha_{n-\taun}<\alpha_n,$ for $n\geq \nstr\geq \nstr^{(3)}.$

Now, we add the bounds in~\eqref{e: bound.T11},~\eqref{e: bound.T21}, and~\eqref{e: bound.T31}, and plug the bound~\eqref{e: bound.mix.diff}, giving us
\begin{multline}
    \alpha_n\bE T_n
    \leq \alpha^2_n (2SA)L_\lambda\|r_n-\ones\|^2_\infty 
    \\
    + (SA)\left(4L_\lambda+ 5\right)\tau^2_n\alpha^2_n\left(\frac{16C^2_u}{C^2_\ell}\right),
\end{multline}
where we additionally use $1\leq \taun,$ and $72\alpha^4_n<\alpha^2_n$ since $\alpha_n\leq 1/9,$ for $n\geq \nstr\geq \nstr^{(2)}.$

Finally, we substitute $\alpha_n\bE T_n$ with the above bound in the recursive relation~\eqref{e: error.rn.recurse}, giving us the desired contractive recursion
\begin{align}\label{e: onestep.recursion.r}
    & \bE\VL(r_{n+1}-\ones) 
    \leq \big( 1 -c_\lambda\alpha_n \big)\bE\VL(r_n-\ones) 
    \nonumber\\
    & \qquad + \left(\frac{2L_\lambda C^2_u}{C^2_\ell}\right) \alpha^2_n
    + \alpha^2_n (2SA)L_\lambda\|r_n-\ones\|^2_\infty 
    \nonumber\\
    & \qquad + (SA)\left(4L_\lambda+ 5\right)\tau^2_n\alpha^2_n\left(\frac{16C^2_u}{C^2_\ell}\right)
    \nonumber\\
    & \overset{(a)}{\leq} \big( 1 -c_\lambda\alpha_n \big)\bE\VL(r_n-\ones) 
    \nonumber\\
    & \qquad + (SA)\left(66L_\lambda+ 80\right)\left(\frac{C^2_u}{C^2_\ell}\right) \taun^2\alpha^2_n
    \nonumber\\
    & \qquad\qquad + \alpha^2_n (2SA)L_\lambda\|r_n-\ones\|^2_\infty 
    \nonumber\\
    & \overset{(b)}{\leq} \big( 1 -c_\lambda\alpha_n \big)\bE\VL(r_n-\ones) 
    \nonumber\\
    & \qquad + (SA)\left(68 L_\lambda+ 80\right)\left(\frac{L_\lambda C^2_u}{C^2_\ell}\right) \taun^2\alpha^2_n
    \nonumber\\
    & = \big( 1 -c_\lambda\alpha_n \big)\bE\VL(r_n-\ones) 
    + C_1 \taun^2\alpha^2_n,
\end{align}
where $(a)$ uses $1\leq SA\taun^2,$ and $(b)$ uses $\ones, r_n\in \cK.$

This completes the proof of Lemma~\ref{lem: recurision.lyapunov}. 
\end{proof}

Lemma~\ref{lem: recurision.lyapunov} establishes a one-step contractive recursion for $\bE\VL(r_n-\ones)$ in the desired form of~\cite[Equation (7)]{chen2021lyapunov}. 

Now, we define for all $\nstr\leq n_1<n_2:$
\begin{equation}\label{def: Gnn}
    \Gamma_{n_1:n_2} := \prod_{n=n_1}^{n_2-1}\big( 1-c_\lambda \alpha_n \big).
\end{equation}

recursively applying~\eqref{e: onestep.recursion.r} yields for $N> \nstr,$
\begin{align}
    & \bE\VL(r_N-\ones) 
    \nonumber\\
    & \leq \Gamma_{\nstr:N}\ \bE\VL(r_{\nstr}-\ones) 
    + C_1\sum_{n=\nstr}^{N-1} \taun^2\alpha^2_n\ \Gamma_{n+1:N}
    \nonumber\\
    & \overset{(a)}{\leq} \left(\frac{4C^2_u}{\ell^2_VC^2_\ell}\right)\Gamma_{\nstr:N} + \tau^2_N \ C_1\sum_{n=\nstr}^{N-1}\alpha^2_n\ \Gamma_{n+1:N},
\end{align}
where $(a)$ uses $\taun\leq \tau_N,$~\eqref{e: lyapunov.equivalent}, and the fact that $r_n,\ones\in \cK.$

We first handle the case when $\alpha_n = \frac{1}{(n+1)^\alpha},$ with $\alpha\in (1/2,1).$ Using Lemma~\ref{lem: Gnn.bound} to bound the iterated products, we get for $N>\nstr,$
\begin{align}
    & \bE\VL(r_N-\ones)
    \nonumber\\
    & \leq \left(\frac{4C^2_u}{\ell^2_VC^2_\ell}\right)e^{\frac{c_\lambda}{1-\alpha}(\nstr+1)^{1-\alpha} }  e^{-\frac{c_\lambda}{1-\alpha}N^{1-\alpha} }
    \nonumber\\
    & \qquad +  C_1C_{\Gamma,1}(c_\lambda,\alpha)\tau^2_N\alpha_N
    \nonumber\\
    &=  \left(\frac{4 C^2_u}{\ell^2_VC^2_\ell}\right) C_{\Gamma,n_*}(c_\lambda,\alpha)e^{-\frac{c_\lambda}{1-\alpha}N^{1-\alpha} }
    \nonumber\\
    & \qquad +  C_1C_{\Gamma,1}(c_\lambda,\alpha)\frac{\tau^2_N}{N^\alpha}
    \nonumber\\
    &\overset{(a)}{\leq} C'_r\frac{\tau^2_N}{N^\alpha}
\end{align}
where to obtain $(a),$ we use the definition of $C'_r$ in Table~\ref{tab:constants-1TS} and use the fact that $y\mapsto y^{2\alpha}\exp\left(-cy^{(1-\alpha)} \right)$ attains maximum at $y=\big[\frac{2\alpha}{c(1-\alpha)}\big]^{1/(1-\alpha)},$ to conclude
\begin{align}\label{e: exp.bound}
    e^{-\frac{c_\lambda}{1-\alpha}N^{1-\alpha} } 
    & = \frac{1}{N^{2\alpha}}\ \Big(N^{2\alpha} e^{-\frac{c_\lambda}{1-\alpha}N^{1-\alpha} }\Big) 
    \nonumber\\
    & \leq \frac{1}{N^{2\alpha}}\left[\frac{2\alpha}{e c_\lambda}\right]^{\frac{\alpha}{1-\alpha}}.
\end{align}
Using~\eqref{e: lyapunov.equivalent} gives us
\begin{align}
    \bE\|r_N-\ones\|_\infty^2 & \leq 2u^2_VC'_r\ \frac{\tau^2_N}{N^\alpha}
    \nonumber\\
    & \overset{(a)}{\leq} \big(C^{(\alpha)}_x\big)^2 \frac{\ln^2(N)}{N^{\alpha}},
\end{align}
where $(a)$ follows from the definition of $C^{(\alpha)}_x$ in Table~\ref{tab:constants-1TS} as we take $\tau_N \leq \frac{\ln(N)\ln(n_* + 1)[\ln (C/\rho) + \alpha]}{\ln^2(n_*)\ln(1/\rho)}.$ 

The claim then follows by applying Jensen's inequality.
%
% \begin{align*}
%     %
%     \bE\VL(r_n-\ones) \leq C'_r\frac{\tau^2_N}{N^\alpha}
%     %
% \end{align*}
% %
% where
% \begin{align*}
%     %
%     C'_r := C_1C_{\Gamma,1} + \left(\frac{2C^2_u}{\ell^2_VC^2_\ell}\right)  \left[\frac{2e\alpha e^{ac_\lambda(k+1)^{1-\alpha}}}{c_\lambda(1-\alpha)}\right]^{\frac{\alpha}{1-\alpha}}
%     %
% \end{align*}

We finish the proof by handling the case when $\alpha_n = \frac{a}{n+1},$ with $ac_\lambda>1.$  In this case, Lemma~\ref{lem: Gnn.bound} gives us
\begin{align}
    & \bE\VL(r_N-\ones)
    \nonumber\\
    & \leq \left(\frac{4C^2_u}{\ell^2_VC^2_\ell}\right)\left(\frac{\nstr+1}{N+1}\right)^{ac_\lambda} +  \frac{C_1C_{\Gamma,1}(c_\lambda,a)\tau^2_N }{N+1}
    \nonumber\\
    & \overset{(a)}{\leq} \left[\left(\frac{4C^2_u}{\ell^2_VC^2_\ell}\right)(2\nstr)^{ac_\lambda} + C_1C_{\Gamma,1}(c_\lambda,a)\right]\frac{\tau^2_N }{N+1},
\end{align}
where $(a)$ uses $1\leq \taun^2,$ $(\nstr+1)\leq 2\nstr,$ and the fact that $ac_\lambda>1.$ Using~\eqref{e: lyapunov.equivalent} gives us
\begin{multline}
    \bE\|r_N-\ones\|^2_\infty 
    \\
    \leq \left[\left(\frac{8\nstr u^2_VC^2_u}{\ell^2_VC^2_\ell}\right) + u^2_VC_1C_{\Gamma,1}(c_\lambda,a)\right]\frac{\tau^2_N }{N+1}.
\end{multline}
Finally, using $\tau_N \leq \frac{\ln(N)\ln(n_* + 1)\ln(\frac{C}{\rho a^2 n_*^2})}{\ln^2(n_*)\ln(1/\rho)}$ and the definition of $C_x$ from Table~\ref{tab:constants-1TS} gives us
\begin{align}
    \bE\|r_N-\ones\|^2_\infty 
    \leq C^2_x\frac{\ln^2(N)}{N+1}.
\end{align}
The claim follows by applying Jensen's inequality. This completes the proof of Theorem~\ref{thm: 1TS}. \hfil\qed

\section{Proof of Theorem~\ref{thm: 1TS.PR}}
\label{appendix: 1TS.PR}

In this section, we provide the detailed proof of Theorem~\ref{thm: 1TS.PR} using the recipe outlined in Section~\ref{s: proof.1TS.PR}.
We begin by introducing some notation.

For each $x\in\bR^{SA},$ let $\pi_x$ denote the greedy policy with respect to $x,$ i.e., for every state $s,$ $\pi_x(s) := \argmin_a x(s,a),$ and defines the row-stochastic matrices $P_x, A^x\in \bR^{SA\times SA}$ as
\begin{equation}
    P_x(s,a,s',a') : = P(s'|s,a)\ones_{\{\pi_x(s')=a'\}},
\end{equation}
and
\begin{equation}
    A^{x}(s,a,s',a') 
    := e^{-\frac{\theta}{\gamma} r(s,a)} P_x(s,a,s',a'),
\end{equation}
for all state-action pairs $(s,a)$ and $(s',a').$ Then, since $y\mapsto y^\gamma$ is increasing on $\bRp,$ we have for all $x\in\bRp^{SA},$
\begin{multline}\label{e: new.operator}
    F(x) = A^xx^\gamma
    \\
    \text{and} \quad \bF(x)=(I-D_\mu)x + D_\mu A^x x^\gamma,
\end{multline}
where the power-law is applied element-wise. 
Borrowing from the recipe used in~\cite{naskar2026parameter}, where the authors analyze Polyak-Ruppert averaging for general nonlinear fixed-point iterations, we 
use the above representation to decompose $\bF$ as a perturbed affine transformation.

This is formalized in the following result.
\begin{lemma}\label{lem: linearize}
    We can decompose the operator $\bar{F}$ as the perturbed affine operator:
    \begin{align*}
         \bF(x) & = \bFlin(x) + D_\mu R(x)+ D_\mu(A^{x} - A^{\xstr})x^\gamma,
    \end{align*}
    such that 
    \begin{align*}
        \left( \big(\p{\bFlin(x)}{\xstr}\big) - \ones \right) = B^*\left(\big(\p{x}{\xstr}\big) -\ones \right),
    \end{align*}
    where $B^*$ is defined as 
    \begin{align*}
        B^* := \big(\bI - D_\mu \big)
        +  \gamma D_\mu \big[ \diag(\xstr)^{-1} A^{\xstr} \diag(\xstr)^{\gamma}\big].
    \end{align*}
    Moreover, $B^*$ satisfies $\|B^*\|_\infty \leq \big( 1- \delta( 1 -\gamma) \big),$
    where $\delta = \min_+ D.$ Further,
    \begin{multline*}
        \big\| A^x-A^{\xstr} \big\|_\infty \leq \Cq\ \left\|\big(\p{x}{\xstr}\big) -\ones \right\|^2_\infty
        \\
        \text{and} \quad \big\| R(x)\big\|_\infty \leq \Cq\ \left\|\big(\p{x}{\xstr}\big) -\ones \right\|^2_\infty,
    \end{multline*}
    where $\Cq>0$ is defined in Table
\end{lemma}
\begin{proof}[\textbf{Proof of Lemma~\ref{lem: linearize}}]
    Since the exponentiations on all vectors are applied component-wise, we can use Taylor expansion for $x^\gamma$ to write
    \begin{align*}
        x^\gamma  & = (\xstr)^\gamma + \gamma\ \diag\big((\xstr)^{\gamma-1}\big)(x-\xstr) 
        \\
        & \quad  - \frac{\gamma(1-\gamma)}{2}\ \diag\big((x')^{\gamma-2}\big)(x-\xstr)^2
        \\
        & = (\xstr)^\gamma + \gamma\ \diag\big((\xstr)^{\gamma}\big)\left(\p{x}{\xstr} - \ones \right)
        \\
        & - \frac{\gamma(1-\gamma)}{2}\ \diag\big((\xstr)^{\gamma})\cdot \p{\xstr}{x'}^{2-\gamma}\big)\left(\p{x}{\xstr} - \ones \right)^2,
    \end{align*}
    where $x' = \xstr + \diag(a_x)(x - \xstr),$ for some $a_x\in [0,1]^{SA}.$ Thus, we have
    \begin{align}
        \bF&(x)
        \nonumber\\
        & = D_\mu A^x x^\gamma + (\bI - D_\mu)x
        \nonumber\\
        & = D_\mu A^{\xstr} x^\gamma + (\bI - D_\mu)x + D_\mu\big(A^x - A^{\xstr}\big)x^\gamma
        \nonumber\\
        & = D_\mu \big[ A^{\xstr} (\xstr)^\gamma\big] + \gamma D_\mu \big[ A^{\xstr} \diag\big((\xstr)^{\gamma}\big) \big]\left(\p{x}{\xstr} - \ones \right)
        \nonumber\\
        & + (\bI - D_\mu)x + D_\mu\big(A^x - A^{\xstr}\big)x^\gamma - \frac{\gamma(1-\gamma)}{2}
        \nonumber\\
        & \quad \times D_\mu \big[A^{\xstr}\diag(\xstr)^{\gamma}\diag(\p{\xstr}{x'})^{2-\gamma}\big]\left(\p{x}{\xstr}-\ones\right)^2 
        \nonumber\\
        & \overset{(a)}{=} D_\mu\xstr + \gamma D_\mu \big[ A^{\xstr} \diag\big((\xstr)^{\gamma}\big) \big]\left( \p{x}{\xstr} - \ones \right)
        \nonumber\\
        & + (\bI - D_\mu)x + D_\mu\big(A^x - A^{\xstr}\big)x^\gamma - \frac{\gamma(1-\gamma)}{2}
        \nonumber\\
        & \quad \times D_\mu \big[A^{\xstr}\diag(\xstr)^{\gamma}\diag(\p{\xstr}{x'})^{2-\gamma}\big]\left(\p{x}{\xstr}-\ones\right)^2
        \nonumber\\
        & \overset{(b)}{=} \bFlin(x) + R(x) +  D_\mu\big(A^x - A^{\xstr}\big)x^\gamma,
    \end{align}
    where $(a)$ follows since $A^{\xstr}(\xstr)^\gamma = F(\xstr)=\xstr,$ and $(b)$ follows as we define
    \begin{multline}
        \bFlin(x) 
        := \gamma D_\mu \big[ A^{\xstr} \diag(\xstr)^{\gamma} \big]\left( \p{x}{\xstr} - \ones \right)
        \\
        + (\bI - D_\mu)(x- \xstr) +  \xstr, 
    \end{multline}
    and
    \begin{align}\label{e: F.nonlin}
        R(x) 
        & := \frac{\gamma(\gamma-1)}{2} D_\mu\big[A^{\xstr}\diag(\xstr)^{\gamma}\diag(\p{\xstr}{x'})^{2-\gamma}\big] \nonumber\\
        & \quad \times 
        \left(\p{x}{\xstr}-\ones\right)^2.
    \end{align}
    Now, we have
    \begin{align}
        & \left( \p{\bFlin(x)}{\xstr} - \ones \right) 
        \nonumber\\
        & = \diag (\xstr)^{-1} \left[ \bFlin(x) -\xstr \right]
        \nonumber\\
        & = \gamma\ \diag(\xstr)^{-1} D_\mu \big[ A^{\xstr} \diag(\xstr)^{\gamma} \big]\left( \p{x}{\xstr} - \ones \right)
        \nonumber\\
        & \quad + \diag (\xstr)^{-1} (\bI - D_\mu)(x-\xstr)
        \nonumber\\
        & \overset{(a)}{=} \gamma D_\mu \big[ \diag (\xstr)^{-1} A^{\xstr} \diag(\xstr)^{\gamma} \big]\left( \p{x}{\xstr} - \ones \right)
        \nonumber\\
        & \quad + (\bI - D_\mu)\ \diag (\xstr)^{-1}(x-\xstr)
        \nonumber\\
        & = \gamma D_\mu \big[ \diag (\xstr)^{-1}  A^{\xstr} \diag(\xstr)^{\gamma}\big]\left( \p{x}{\xstr} - \ones \right) 
        \nonumber\\
        & \quad + (\bI - D_\mu)\left(\p{x}{\xstr} -\ones \right)
        \nonumber\\
        & \overset{(b)}{=} B^*\left(\p{x}{\xstr} -\ones \right),
    \end{align}
    where $(a)$ follows since diaginal matrices commute, and $(b)$ follows as define the matrix
    \begin{multline*}
        B^* 
        := \big(\bI - D_\mu \big)
        +  \gamma D_\mu \big[ \diag (\xstr)^{-1}  A^{\xstr} \diag\big((\xstr)^{\gamma}\big) \big].
    \end{multline*}
    Let $\delta= \min_+ D_\mu.$ Then, we claim that $\|B^*\|_\infty \leq \left( 1 - \delta(1-\gamma) \right).$ To see this, note that 
    \begin{align}
        & \big[ \diag (\xstr)^{-1}  A^{\xstr} \diag\big((\xstr)^{\gamma}\big) \big]\ \ones 
        \nonumber\\
        & \qquad\qquad = \big[\diag (\xstr)^{-1} \big] A^{\xstr} \big( \xstr\big)^\gamma 
        \nonumber\\
        & \qquad\qquad = \big[\diag (\xstr)^{-1} \big] \xstr = \ones.
    \end{align}
    Therefore, the above matrix is row-stochastic. Hence, 
    \begin{equation}\label{e: row.stoc}
        \left\| \big[ \diag(\xstr)^{-1}  A^{\xstr} \diag\big((\xstr)^{\gamma}\big) \big] \right\|_\infty \leq 1.
    \end{equation}
    We can then show that
    \begin{align}
        \|B^*\|_\infty \leq \big( 1 - \delta(1-\gamma) \big).
    \end{align}

    Next, we show that $\|A^x-A^{\xstr}\|_\infty = O\left(\|\p{x}{\xstr} - \ones\|^2_\infty \right).$
    Since the map $x\mapsto\pi_x$ is piece-wise constant,  the map $x\mapsto A^{x}$ is also piece-wise constant. For any assignment $\bar{a}\in \cA^{\cS}$ of actions to states, we define the cone
    \begin{equation}
        \cC_{\bar{a}} := \left\{ x\in \bRp^{\SA} : \pi_x(s) = \bar{a}(s), \text{ for all } s\in \cS\right\}.
    \end{equation}
    Then, the cones $\{\cC_{\bar{a}}: \bar{a}\in \cA^{\cS}\}$ cover $\bRp^{\SA}$ and have disjoint interiors. Since $\pi_{\xstr}$ assigns unique actions to each state, we know that $\xstr$ does not lie on the boundary of any cone. Let the cone containing $\xstr$ be denoted by $\cC_*.$ Then, we define 
    \begin{align}   
        c^* := \frac{1}{2}\max\big\{ r>0 : B(\xstr, r)\subset \cC_* \big\},
    \end{align}
    where $B(\xstr,r)$ denotes the open ball of radius $r$ centered at $\xstr.$ It is clear that if $\|x-\xstr\|_\infty\leq c_*,$ then $x\in \cC_*$ and thus, $A^x = A^{\xstr}.$ Thus, we have
    \begin{align}\label{kk}
        \big\| A^x - A^{\xstr} \big\|_\infty  
        & \leq e^{\frac{\theta}{\gamma}\Rm}\cdot\ones_{\{\|x-\xstr\|_\infty > c_* \}}
        \nonumber\\
        & \leq \frac{1}{c^2_*}\ e^{\frac{\theta}{\gamma}\Rm}\ \| x- \xstr \|^2_\infty 
        \nonumber\\
        & \overset{(a)}{\leq} \frac{C_u^2}{ c^2_*}e^{\frac{\theta}{\gamma}\Rm}\  \left\| \p{x}{\xstr} - \ones\right\|^2_\infty,
    \end{align}
    where $(a)$ uses the fact that $\xstr \in [C_u,C_\ell]^{\SA}.$

    Finally, we show that $\|R(x)\|_\infty = O\left(\|\p{x}{\xstr} - \ones \|^2_\infty \right).$ We get from~\eqref{e: F.nonlin}, that
    \begin{align}\label{y}
        \big\|R(x)\big\|_\infty 
        & \overset{(a)}{\leq} \left \|\big[A^{\xstr}\diag\big((\xstr)^{\gamma}\big)\diag\big((\p{\xstr}{x'})^{2-\gamma}\big)\big]\right\|_\infty 
        \nonumber\\
        & \qquad\times\left\|\left(\p{x}{\xstr}-\ones\right)^2 \right\|_\infty
        \nonumber\\
        & \overset{(b)}{\leq}   \frac{C^3_u}{C^2_\ell}\left \|\big[ \diag\big( (\xstr)^{-1}\big)A^{\xstr}\diag\big((\xstr)^{\gamma}\big)\big]\right\|_\infty 
       \nonumber \\
        & \qquad \times\left\|\left(\p{x}{\xstr}-\ones \right)^2\right\|_\infty
        \nonumber\\
        & \overset{(c)}{\leq}  \frac{C^3_u}{C^2_\ell} \left\|\left(\p{x}{\xstr}-\ones \right)^2\right\|_\infty
        \nonumber\\
        & \overset{(d)}{\leq}  \frac{C^3_u}{C^2_\ell} \left\|\p{x}{\xstr}-\ones \right\|^2_\infty,
    \end{align}
    where $(a)$ follows since $\gamma\in (0,1)$ and $\|D\|_\infty \leq 1,$ $(b)$ is obtained by multiplying by $\bI = \big[\diag\big(\xstr\big)\big]  \diag(\xstr)^{-1},$ and then using $\|\diag( \xstr)\|_\infty$ $=\|\xstr\|_\infty$ $\leq C_u$ and $\|\diag(\p{\xstr}{x'})^{2-\gamma})\|_\infty $ $= \|\p{\xstr}{x'}\|^2_\infty$ $\leq C^2_u/C^2_\ell,$ $(c)$ follows from~\eqref{e: row.stoc}, and $(d)$ follows since exponentiation on vectors is applied component-wise. 
     
    This completes the proof of Lemma~\ref{lem: linearize}.
\end{proof}
Lemma~\ref{lem: linearize} allows us to rewrite the update~\eqref{e: 1TS.ratio} as a perturbed stochastic linear recursion:
\begin{align}\label{e: perturbed.linear}
    \left(r_{n+1}- \ones\right) 
    & = \big[ I - \alpha_n (I - B^* )\big] \left( r_n - \ones \right)
    \nonumber\\
    & + \alpha_n\tF(x_n,s_n,a_n,s_{n+1}) + \alpha_n \xi_n,
\end{align}
where $\xi_n$ denotes the nonlinear perturbation defined as 
\begin{equation}\label{e: perturbation.term}
    \xi_n := \diag\big(\xstr \big)^{-1} D_\mu\left( R(x_n) + \big[ A^{x_n} - A^{\xstr}\big] x^\gamma_n \right),
\end{equation}
and $\tF(x_n,s_n,a_n,s_{n+1})$ is Markovian noise as defined in~\eqref{e: noise}. 

Unrolling the above recursion~\eqref{e: perturbed.linear}, we get for $n>\nstr,$
\begin{align}\label{e: Delta.decomp}
    \big(r_n-\ones\big) & = G_{n_*:n}\big(r_{n_*}-\ones\big) 
    \nonumber\\
    & \quad + \sum_{k=n_*}^{n-1}\alpha_n G_{k+1:n}\tF(x_k,s_k,a_k,s_{k+1}) 
    \nonumber\\
    & \quad + \sum_{k=n_*}^{n-1} \alpha_n G_{k+1:n} \xi_k,
\end{align}
where, for $\nstr\leq n_1<n_2,$ we define the following iterated matrix product:
\begin{equation}
    G_{n_1:n_2} := \prod_{n=n_1}^{n_2-1}\big[ I - \alpha_n (I - B^* )\big].
\end{equation}

Now, if we define
\[
    \Delta_n = r_n-\ones, \quad \text{and} \quad \bDelta_n:= \left( \p{\bar{x}_n}{\xstr} -\ones\right),
\]
Then, $\bDelta_N= \frac{1}{N}\sum_{n=0}^{N-1}\Delta_n,$ for all $N>0.$ Combining this with the decomposition in~\eqref{e: Delta.decomp} gives for all $N>\nstr,$

\begin{align}\label{e: PR.decomp} 
    \bDelta_N = \bDeltainit_N + \bDeltatrans_N + \bDeltalin_N + \bDeltanlin_N,
\end{align}
where
\begin{align*}
    & \bDeltainit_N := \frac{1}{N}\sum_{n=0}^{n_*-1} \Delta_n,\quad  \bDeltatrans_N := \frac{1}{N}\sum_{n=n_*}^{N-1} G_{n_*:n} \Delta_{n_*}
    \\
    & \quad\qquad \bDeltalin_N := \frac{1}{N}\sum_{n=n_*}^{N-1}\sum_{k=n_*}^{n-1}\alpha_k G_{k+1:n}\tF(x_{k}, z_k),
    \\
    & \quad\qquad \bDeltanlin_N := \frac{1}{N}\sum_{n=n_*}^{N-1}\sum_{k=n_*}^{n-1}\alpha_k G_{k+1:n}\xi_k,
\end{align*}
with $z_k$ denoting the tuple $(s_k,a_k,s_{k+1}),$ for all $k.$

To prove Theorem~\ref{thm: 1TS.PR}, we bound each term in~\eqref{e: PR.decomp}.

\paragraph{Bounding $\bDeltainit_N$:} The first term in~\eqref{e: PR.decomp} is bounded in Lemma~\ref{lem: PR.init}. Since $r_n,\ones\in\cK,$ we use triangle inequality to get
\begin{multline}
    \|\bDeltainit_N\|_\infty \leq \frac{1}{N}\sum_{n=0}^{\nstr-1}\|\Delta_n\|_\infty
    \\
    \leq \frac{1}{N}\sum_{n=0}^{\nstr-1}\frac{2C_u}{C_\ell}
    = \frac{2 C_u \nstr}{C_\ell N} =: \frac{C^\textnormal{init}}{N}.
\end{multline}
Lemma~\ref{lem: PR.init} follows as we take expectation. \hfill\qed

\paragraph{Bounding $\bDeltatrans_N$:} Lemma~\ref{lem: PR.trans} bounds the second term in~\eqref{e: PR.decomp}. Note that, since $r_n,\ones\in\cK,$ applying triangle inequality on $\bDeltatrans_N$ gives 
\begin{align}
    \|\bDeltatrans_N\|_\infty
    & \leq \frac{1}{N}\sum_{n=n_*}^{N-1} \|G_{n_*:n}\|_\infty \|\Delta_{n_*}\|_\infty
    \nonumber\\
    & \leq \left(\frac{2C_u}{C_\ell N}\right) \sum_{n=n_*}^{N-1} \|G_{n_*:n}\|_\infty
    \nonumber\\
    & \overset{(a)}{\leq} \left(\frac{2C_u}{C_\ell N}\right)\sum_{n=n_*}^{N-1} e^{\frac{\delta(1-\gamma)}{1-\alpha}(\nstr+1)^{1-\alpha} } e^{-\frac{\delta(1-\gamma)}{1-\alpha}n^{1-\alpha} }
    \nonumber\\
    & \overset{(b)}{\leq} \left(\frac{2C_u}{C_\ell N}\right)e^{\frac{2\nstr\delta(1-\gamma)}{1-\alpha}}\sum_{n=\nstr}^{N-1}\frac{1}{n^2}\left[\frac{2}{e\delta(1-\gamma)}\right]^{\frac{\alpha}{1-\alpha}}
    \nonumber\\
    & \overset{(c)}{\leq}\left[\frac{2}{e\delta(1-\gamma)}\right]^{\frac{\alpha}{1-\alpha}} \left(\frac{\pi^2 C_u}{3C_\ell N}\right)e^{\frac{2\nstr\delta(1-\gamma)}{1-\alpha}}
    \nonumber\\
    & \overset{(d)}{\leq} \frac{C^\textnormal{trans}}{N},
\end{align}
where $(a)$ follows from Lemma~\ref{lem: Gnn.bound}, $(b)$ uses $(\nstr+1)^{(1-\alpha)} < 2\nstr $ and the fact that $y\mapsto y^2 e^{-\frac{c}{1-\alpha} y^{1-\alpha}}$ attains maximum at $y=(2/c)^{1/(1-\alpha)},$ and $(c)$ uses $\sum_n 1/n^2 \leq \pi^2/6.$ Finally, $(d)$ uses the definition of $C^\textnormal{trans}$ from Table~\ref{tab:constants-1TS}. Lemma~\ref{lem: PR.trans} follows as we take expectation. \hfill\qed

\paragraph{Bounding $\bDeltalin_N$:} We now bound the averaged error $\bDeltalin_N$ arising from the Markovian noise $\tF(x_k,z_k).$ Note that
\begin{align}
    \bDeltalin_N & = \frac{1}{N}\sum_{n=n_*}^{N-1}\sum_{k=n_*}^{n-1}\alpha_k G_{k+1:n}\tF(x_{k}, z_k)
    \nonumber\\
    & = \frac{1}{N}\sum_{k=n_*}^{N-2}\left(\alpha_k\sum_{n=k+1}^{N-1} G_{k+1:n}\right)\tF(x_{k}, z_k)
    \nonumber\\
    & = \frac{1}{N}\sum_{k=n_*}^{N-2}S_{k:N}\tF(x_{k}, z_k),
\end{align}
where $S_{k:N}$ is as defined in Lemma~\ref{lem: Gnn.bound}. Squaring both sides and taking expectation, we have
\begin{align}\label{e: decomp.inner.prod}
    & \bE\|\bDeltalin_N\|^2_2 
    \nonumber\\
    &= \frac{1}{N^2}\sum_{k=\nstr}^{N-2}\bE\|S_{k:N}\tF(x_k,z_k)\|^2_2
    \nonumber\\
    & + \frac{2}{N^2}\sum_{t=\nstr+1}^{N-2}\sum_{s=\nstr}^{t-1}\bE\left\langle S_{s:N}\tF(x_s,z_s), S_{t:N}\tF(x_t,z_t)\right\rangle
    \nonumber\\
    & = \frac{1}{N^2}\sum_{k=\nstr}^{N-2}\bE\|S_{k:N}\tF(x_k,z_k)\|^2_2
    \nonumber\\
    & + \frac{2}{N^2}\sum_{t=\nstr+1}^{N-2}\sum_{s=t-2\tau_t}^{t-1}\bE\left\langle S_{s:N}\tF(x_s,z_s), S_{t:N}\tF(x_t,z_t)\right\rangle
    \nonumber\\
    & + \frac{2}{N^2}\sum_{t=\nstr+1}^{N-2}\sum_{s=\nstr}^{t-2\tau_t-1}\bE\left\langle S_{s:N}\tF(x_s,z_s), S_{t:N}\tF(x_t,z_t)\right\rangle.
\end{align}
The first term on the r.h.s can be bounded as follows:
\begin{align}\label{e: noise.1}
    \frac{1}{N^2}& \sum_{k=\nstr}^{N-2}\bE\|S_{k:N}\tF(x_k,z_k)\|^2_2
    \nonumber\\
    & \overset{(a)}{\leq} \frac{SA}{N^2}\sum_{k=\nstr}^{N-2}\|S_{k:N}\|^2_\infty\bE\|\tF(x_k,z_k)\|^2_\infty
    \nonumber\\
    & \overset{(b)}{\leq} \frac{SAK_G^2}{N^2}\sum_{k=\nstr}^{N-2}\left(\frac{6C_u}{C_\ell}\right)^2
    \nonumber\\
    & \leq \left(\frac{6C_u}{C_\ell}\right)^2\frac{SAK_G^2}{N}
\end{align}
where $(a)$ uses $\|\cdot\|_2\leq SA\|\cdot\|_\infty,$ $(b)$ follows from Lemma~\ref{lem: Gnn.bound} and since $\|\tF(x_k,z_k)\|_\infty\leq 6C_u/C_\ell.$

Next, we bound the second term on the r.h.s of~\eqref{e: decomp.inner.prod} as follows
\begin{align}\label{e: noise.2}
    \frac{2}{N^2}& \sum_{t=\nstr+1}^{N-2}\sum_{s=t-2\tau_t}^{t-1}\bE\left\langle S_{s:N}\tF(x_s,z_s), S_{t:N}\tF(x_t,z_t)\right\rangle
    \nonumber\\
    & \overset{(a)}{\leq } \frac{2}{N^2}\sum_{t=\nstr+1}^{N-2}\sum_{s=t-2\tau_t}^{t-1}36K^2_G\left( \frac{C^2_u}{C^2_\ell}\right)
    \nonumber\\
    & \leq 144\left( \frac{K^2_G C^2_u}{C^2_\ell}\right)\frac{\tau_t}{N},
\end{align}
where $(a)$ combines the Cauchy-Schwarz lemma with Lemma~\ref{lem: Gnn.bound} and the inequality $\|\tF(x_k,z_k)\|_\infty\leq 6C_u/C_\ell.$

Finally, we bound the last term as follows
\begin{align*}
    & 2\bE\left\langle S_{s:N}\tF(x_s,z_s), S_{t:N}\tF(x_t,z_t)\right\rangle
    \nonumber\\
    & = 2\bE\left\langle S_{s:N}\tF(x_s,z_s), S_{t:N}\ \bE_{t-2\tau_t}\tF(x_{t-2\tau_t},z_t)\right\rangle
    \nonumber\\
    & + 2\bE\Big\langle S_{s:N}\ \bE_{s-\tau_s}\tF(x_{s-2\tau_s},z_s), \nonumber\\
    & \qquad S_{t:N}\left( \tF(x_{t},z_t)- \tF(x_{t-2\tau_t},z_t) \right)\Big\rangle
    \nonumber\\
    & + 2\bE\Big\langle S_{s:N}\left( \tF(x_s,z_s)-\tF(x_{s-2\tau_s},z_s)\right), 
    \nonumber\\
    & \qquad S_{t:N}\left( \tF(x_{t},z_t)- \tF(x_{t-2\tau_t},z_t) \right)\Big\rangle
    \nonumber\\
    & \overset{(a)}{\leq }  (SAK_G)^2\alpha^2_s\alpha^2_t\ \bE\|\tF(x_s,z_s)\|^2_\infty \nonumber\\
    & + \frac{(SA K_G)^2}{\alpha^2_s\alpha^2_t}\bE\|\bE_{t-2\tau_t} \tF(x_{t-2\tau_t},z_t)\|^2_\infty
     \nonumber\\
    & +  (SAK_G)^2\ \bE \|\bE_{s-2\tau_s}\tF(x_{s-2\tau_s}, z_s)\|^2_\infty
    \nonumber\\
    & + 2(SAK_G)^2\ \bE\|\tF(x_{t},z_t)- \tF(x_{t-2\tau_t},z_t)\|^2_\infty
     \nonumber\\
     & + (SAK_G)^2\ \bE\|\tF(x_{s},z_s)- \tF(x_{s-2\tau_s},z_s)\|^2_\infty
     \nonumber\\
     & \overset{(b)}{\leq} \left(\frac{6SAK_GC_u}{C_\ell}\right)^2 \bigg(\alpha^2_s\alpha^2_t + \frac{\alpha^8_t}{\alpha^2_s\alpha^2_t} + \alpha_s^8 
     \nonumber\\
     & + 2\tau^2_{t-2\tau_t}\alpha^2_{t-2\tau_t} + \tau^2_{s-2\tau_s}\alpha^2_{s-2\tau_s}\bigg)
\end{align*}
where $(a)$ uses Young's inequality and Lemma~\ref{lem: Gnn.bound}, and $(b)$ follows from Assumption~\ref{a: ergodicity}, and uses $C\rho^{2\tau_t}<\alpha^4_t$ and $C\rho^{2\tau_s}<\alpha^4_s,$ and $\tau_{t-2\tau_t},\tau_{t-2\tau_s}\leq \tau_N.$ 

Thus,
\begin{align}\label{e: noise.3}
    & \frac{2}{N^2}\sum_{t=\nstr+1}^{N-2}\sum_{s=\nstr}^{t-2\tau_t-1}\bE\left\langle S_{s:N}\tF(x_s,z_s), S_{t:N}\tF(x_t,z_t)\right\rangle
    \nonumber\\
    & \overset{(a)}{\leq} \left(\frac{6SAK_GC_u}{N C_\ell}\right)^2
    \bigg[\sum_{t=\nstr+1}^{N-2}\sum_{s=\nstr}^{t-2\tau_t-1}2\alpha^2_s\alpha^2_t 
    \nonumber\\
    & + \sum_{t=\nstr+1}^{N-2}\sum_{s=\nstr}^{N-1}\left( \alpha^8_s + 2\tau^2_N \alpha^2_{t-2\tau_t} + \tau^2_N \alpha^2_{s-2\tau_s}\right) \bigg]
    \nonumber\\
    & \overset{(b)}{\leq} \left(\frac{6SAK_GC_u}{N C_\ell}\right)^2\bigg[2C_\alpha^2 + C_\alpha + 3C_\alpha\tau^2_N \bigg]N
    \nonumber\\
    & \overset{(c)}{\leq} \left(\frac{6SAK_GC_u}{ C_\ell}\right)^2 \frac{2(C_\alpha+1)^2\tau^2_N}{N}
\end{align}
where $(a)$ follows since $\alpha^8_t \leq \alpha^4_s\alpha^4_t,$ and to obtain $(b)$ we use the fact that $C_\alpha:=\sum_n\alpha^2_n<\infty$ since $\alpha>1/2.$ Finally, $(c)$ uses $1\leq N,\tau_N.$

Lemma~\ref{e: PR.noise} follows as we combine~\eqref{e: noise.1},~\eqref{e: noise.2}, and~\eqref{e: noise.3} with~\eqref{e: decomp.inner.prod}. \hfill\qed

\paragraph{Bounding $\bDeltanlin_N$:} The bound for $\bDeltanlin_N$ is given in Lemma~\ref{lem: PR.nlin}. 
Note that the nonlinear perturbation $\xi_k$ can be bounded as
\begin{align}\label{e: xi.bound}
    \bE\|\xi_k\|_\infty & \overset{(a)}{\leq} \frac{1}{C_\ell} \left(\bE\|R(x_k)\|_\infty + C_u\bE\| A^{x_n} - A^{\xstr}\|_\infty\right)
    \nonumber\\
    & \overset{(b)}{\leq} \frac{2\Cq}{C_\ell}(1+C_u)\bE\|r_n-\ones\|^2_\infty
    \nonumber\\
    & \overset{(c)}{\leq} \frac{2\Cq}{C_\ell}(1+C_u)C^{(\alpha)}_x \frac{\tau^2_k}{k^\alpha}
\end{align}
where $(a)$ follows from~\eqref{e: perturbation.term} and since $\xstr\in\cK$ and $\|D_\mu\|_\infty\leq 1,$ $(b)$ follows from Lemma~\ref{lem: linearize}, and $(c)$ follows from Theorem~\ref{thm: 1TS}.

Now, applying the triangle inequality, we get
\begin{align*}
    \bE\|\bDeltanlin_N\|_\infty & \leq \frac{1}{N}\sum_{k=\nstr}^{N-2} \|S_{k:N}\|_\infty \bE\|\xi_k\|_\infty
    \\
    & \overset{(a)}{\leq} \frac{K_G}{N}\sum_{k=\nstr}^{N-2} \bE\|\xi_k\|_\infty
    \\
    & \overset{(b)}{\leq} \frac{2\Cq}{C_\ell}(1+C_u)C^{(\alpha)}_x \frac{K_G}{N}\sum_{k=\nstr}^{N-2} \frac{\tau_k^2}{k^\alpha}
    \\
    & \overset{(c)}{\leq}\frac{2\Cq}{C_\ell}\frac{(1+C_u)}{1-\alpha}K_G C^{(\alpha)}_x \frac{\tau^2_N}{N^\alpha}
    \nonumber\\
    & = C^\textnormal{nonlin}\frac{\tau^2_N}{N^\alpha},
\end{align*}
where $(a)$ from Lemma~\ref{lem: Gnn.bound}, $(b)$ uses~\eqref{e: xi.bound}, and $(c)$ follows since $\tau_k\leq \tau_N$ and $\sum_{k=\nstr}^{N-2}1/k^\alpha < N^{1-\alpha}/(1-\alpha).$ This finishes the proof of Lemma~\ref{lem: PR.nlin}.\hfill\qed

Now, we combine the bounds established in Lemma~\ref{lem: PR.init}~\ref{lem: PR.trans},~\ref{lem: PR.noise}, and~\ref{lem: PR.nlin} with the decomposition~\eqref{e: PR.decomp} to obtain for $N>\nstr,$
\begin{align*}
    & \bE\|\bDelta_N\|_\infty 
    \\
    & \leq  \left[C^{\textnormal{init}} + C^\textnormal{trans}\right]\frac{1}{N} + \frac{\tau_NC^{\textnormal{noise}}}{\sqrt N} + \frac{\tau^2_NC^{\textnormal{nonlin}}}{N^\alpha}
    \\
    & \leq \left[C^{\textnormal{init}} + C^\textnormal{trans} + C^{\textnormal{noise}} + C^{\textnormal{nonlin}}\right]\frac{\tau_N}{\sqrt N},
\end{align*}
where the last inequality follows since $1<\tau^2_N$ and since $\alpha>1/2$ and $\tau_N= \cO(\ln N)$ implies $\tau_N/N^{(\alpha-1/2)} < 1.$ 

Theorem~\ref{thm: 1TS.PR} follows as we use $\tau_N\leq \frac{\ln(N)\ln(n_* + 1)\ln(\frac{C}{\rho a^2 n_*^2})}{\ln^2(n_*)\ln(1/\rho)}.$ \hfill\qed

\section{Proof of Theorem~\ref{thm: 2TS}}
\label{appendix: 2.TS}

In this section, we provide the detailed proof of Theorem~\ref{thm: 2TS}. The $(g_n)$ iterates are updated on a faster timescale compared to $(Q_n).$ Hence, we first analyze $g_n$'s convergence considering a fixed value for $Q_n,$ and then plug its corresponding limit into the slower $(Q_n)$ update. 
%
% Since the faster iterate $g_n$ tracks a moving target, this results in an additional tracking error as follows:
% %
% \begin{equation}\label{e: Q.update}
%     %
%     Q_{n + 1} = Q_n + \beta_n[ TQ_n - Q_n ] + \beta_n\epsp_n,
%     %
% \end{equation}
% %
% where $\epsp_n :=  \left[- \frac{\gamma}{\theta} \ln g_n - TQ_n \right]$ denotes the faster iterate's tracking error. Since $T$ is a $\gamma$-contraction under $\|\cdot\|_\infty,$ the update rule~\eqref{e: Q.update} consists of a contractive drift perturbed by the tracking error $\epsp_n.$ Using simple inductive arguments,~\cite[Theorem B.1.]{thoppe2026reinforcement} showed that the iterates $(Q_n, \frac{\gamma}{\theta}\ln g_n)$ remain confined to a compact set. In particular, $(g_n)$ remains bounded away from zero. Using the Lipschitz continuity of $\ln(\cdot)$ on this compact set, we have $\|\epsp_n\|_\infty=\cO(\|g_n -e^{-\frac{\theta}{\gamma}TQ_n}\|_\infty).$ Consequently, we control the perturbation term by analyzing the convergence of $y_n:=\big(g_n -e^{-\frac{\theta}{\gamma}TQ_n}\big)$ to zero.

\paragraph{Analyzing $(g_n).$:} To track $g_n$'s the moving target, define $y_n:=g_n - e^{-\frac{\theta}{\gamma} TQ_n}.$ Then, letting $z_n$ denote the tuple $(s_n,a_n,s_{n+1}),$ we have for all $n\geq \nstr,$
\begin{align}\label{e: new.gn.update}
    y_{n+1} 
    &= g_{n+1} - e^{-\frac{\theta}{\gamma} TQ_{n+1}}
    \nonumber\\
    & \overset{(a)}{=} y_n +  \big[e^{-\frac{\theta}{\gamma} TQ_n} - e^{-\frac{\theta}{\gamma} TQ_{n + 1}}\big] 
    \nonumber\\
    & \quad + \alpha_n e_{s_n, a_n}\left[\hG(Q_n,z_n) - g_n(s_n, a_n)\right] 
    \nonumber\\
    & \overset{(b)}{=} (I-\alpha_n D_\mu)y_n +  \big[e^{-\frac{\theta}{\gamma} TQ_n} - e^{-\frac{\theta}{\gamma} TQ_{n + 1}}\big] 
    \nonumber\\
    & \quad + \alpha_n \left[\hG(Q_n,z_n) e_{s_n, a_n} - g_n(s_n, a_n)\right] 
    \nonumber\\
    & \quad + \alpha_nD_\mu \big[g_n - e^{-\frac{\theta}{\gamma}TQ_n}\big]
    \nonumber\\
    & \overset{(c)}{=} (I-\alpha_n D_\mu)y_n + \alpha_n \tG(Q_n,g_n, z_n) + \epsilon_n,
\end{align}
where $(a)$ follows from the update~\eqref{e: 2TS} and by adding and subtracting $e^{-\frac{\theta}{\gamma} TQ_n},$ $(b)$ follows by adding and subtracting $D_\mu\ e^{-\frac{\theta}{\gamma} TQ_n},$ and $(c)$ follows as we define
\begin{multline}\label{e: noise.2TS}
    \tG(Q,g,s,a,s'):= \left(D_\mu- e_{s,a}\ e^\top_{s,a}\right)g \\
    + \left[ \hG(Q,s,a,s')\ e_{s,a} - D_\mu\ e^{-\frac{\theta}{\gamma}TQ} \right].
\end{multline}

Squaring both sides of~\eqref{e: new.gn.update} and taking expectation gives us 
\begin{align}\label{e: squared.g}
    \bE\|y_{n+1}\|^2_2 
    & = \bE\|(I-\alpha_nD_\mu)y_n\|^2_2 
    \nonumber\\
    & + \alpha^2_n \bE\|\tG(Q_n,g_n, z_n)\|^2_2 + \bE\|\epsilon_n\|^2_2
    \nonumber\\
    &  + 2\alpha_n\bE\big\langle(I-\alpha_nD_\mu)y_n, \tG(Q_n,g_n, z_n) \big\rangle 
    \nonumber\\
    &  + 2\alpha_n\bE\big\langle(I-\alpha_nD_\mu)y_n, \epsilon_n \big\rangle 
    \nonumber\\
    & + 2\alpha_n\bE\big\langle \tG(Q_n,g_n, z_n), \epsilon_n\big\rangle 
    \nonumber\\
    & \overset{(a)}{\leq} \big(1-2\delta\alpha_n + \alpha^2_n\big)\|y_n\|^2_2  
    \nonumber\\
    & + 2\alpha^2_n \bE\|\tG(Q_n,g_n, z_n)\|^2_2 + 2\bE\|\epsilon_n\|^2_2
    \nonumber\\
    & + 4e^{\frac{\theta}{\gamma}\Cm}\alpha_n\|\epsp_n\|_2 
    \nonumber\\
    & + 2\alpha_n\bE\big\langle(I-\alpha_nD_\mu)y_n, \tG(Q_n,g_n, z_n) \big\rangle 
\end{align}
where $(a)$ follows since from Cauchy-Schwarz inequality and using $\|I-\alpha_n D\|^2_2 \leq \max_+(I-\alpha_n D)^2.$

It follows from~\cite[Lemma B.1]{thoppe2026reinforcement} that %
\begin{equation}\label{e: Q.bound}
    Q_n,\frac{\gamma}{\theta}\ln g_n\in \cK':= [-\Cm,\Cm]^{SA},
\end{equation}
where $\Cm:= (1+b+\ln\nstr)\frac{\Rm}{1-\gamma}.$ Using the Lipschitzness of $u\mapsto e^{-\theta u},$ can bound the tracking error $\epsilon_n$ as follows:
\begin{align}\label{e: eps.bound}
    \|\epsilon_n\|_\infty & \overset{(a)}{\leq} \frac{\theta}{\gamma} e^{\frac{\theta}{\gamma}\Cm}\ \|TQ_n - TQ_{n+1}\|_\infty
    \nonumber\\
    & \overset{(b)}{\leq} \theta  e^{\frac{\theta}{\gamma}\Cm} \| Q_{n+1} - Q_n\|_\infty
    \nonumber\\
    & \overset{(c)}{\leq} \theta e^{\frac{\theta}{\gamma}\Cm} \beta_n\| \frac{\gamma}{\theta}\ln g_n + Q_n\|_\infty
    \nonumber\\
    & \overset{(d)}{\leq} 2\Cm e^{\frac{\theta}{\gamma}\Cm} \beta_n,
\end{align}
where $(a)$ follows since $u\mapsto e^{-\frac{\theta}{\gamma}u}$ is $\left(\frac{\theta}{\gamma} e^{\frac{\theta}{\gamma}\Cm}\right)$-Lipschitz on $[-\Cm,\Cm],$ $(b)$ follows since $T$ is a $\gamma$-contraction in $\|\cdot\|_\infty,$ $(c)$ follows from the update rule~\eqref{e: 2TS}, and $(d)$ follows from~\eqref{e: Q.bound}.

Moreover, it follows from~\cite[Equation 26]{thoppe2026reinforcement}, that $\|\tG(Q_n, g_n, z_n)\|_\infty\leq 4 e^{\frac{\theta}{\gamma}\Cm}.$ Plugging this bound and~\eqref{e: eps.bound} into~\eqref{e: squared.g} and taking $\beta_n\leq \alpha_n,$ we obtain
\begin{align}\label{e: yn.intermediate}
    & \bE\|y_{n+1}\|^2_2 
    \nonumber\\
    & \leq \big(1-\delta\alpha_n\big)\|y_n\|^2_2
    + (4SA)\ e^{\frac{2\theta}{\gamma}\Cm}(8 + 3\Cm^2 )\ \alpha^2_n
    \nonumber\\
    & + \alpha_n\bE\underbrace{2\big\langle(I-\alpha_nD_\mu)y_n, \bE_n\big[\tG(Q_n,g_n, z_n) \big]\big\rangle}_{T_n},
\end{align}
where we additionally use $\|\cdot\|_2\leq \sqrt{SA}\|\cdot\|_\infty.$ It remains to bound the inner-product
\begin{multline}
    T_n := 2\big\langle(I-\alpha_nD_\mu)y_n, \tG(Q_n,g_n, z_n) \big\rangle 
    \\
    = T_{n,1} + T_{n,2} + T_{n,3},
\end{multline}
where we define
\begin{align*}
    T_{n,1} & := 2\big\langle(I-\alpha_nD_\mu)y_{n-\taun}, \tG(Q_{n-\taun},g_{n-\taun}, z_n) \big\rangle 
    \nonumber\\
    T_{n,2} & := 2\big\langle(I-\alpha_nD_\mu)(y_n-y_{n-\taun}), \tG(Q_n,g_n, z_n) \big\rangle
    \nonumber\\
    T_{n,3} & := 2\big\langle(I-\alpha_nD_\mu)(y_{n-\taun}),
    \nonumber\\
    & \quad \times \big(\tG(Q_n,g_n, z_n) - \tG(Q_{n-\taun},g_{n-\taun}, z_n) \big)\big\rangle.
\end{align*}
Then, we have
\begin{align*}
    & \bE T_{n,1}
    \nonumber\\
    & \overset{(a)}{=} \bE\big\langle(I-\alpha_nD_\mu)y_{n-\taun}, \bE_{n-\taun}\tG(Q_{n-\taun},g_{n-\taun}, z_n) \big\rangle
    \nonumber\\
    & \overset{(b)}{\leq} \alpha_n\bE\|y_{n-\taun}\|^2_2 
    \nonumber \\
    & \qquad+ \frac{1}{\alpha_n}\bE\|\bE_{n-\taun}\tG(Q_{n-\taun},g_{n-\taun}, z_n)\|^2_2
    \nonumber\\
    & \overset{(c)}{\leq} \alpha_n\bE\|y_{n-\taun}\|^2_2 + \frac{16SA}{\alpha_n}\ e^{\frac{2\theta}{\gamma}\Cm} \big(C\rho^{\taun}\big)^2
    \nonumber\\
    & \overset{(d)}{\leq} 2\alpha_n\bE\|y_n\|^2_2 + 2\alpha_n\bE\|y_n - y_{n-\taun}\|^2_2 
    + 16SA e^{\frac{2\theta}{\gamma}\Cm} \alpha^3_n,
\end{align*}
where $(a)$ follows by Tower property of expectation and since $y_{n-\taun}$ is $\cF_{n-\taun}$-measurable, $(b)$ uses Young's inequality and the fact that $\|I-\alpha_nD_\mu\|_2\leq 1,$ $(c)$ follows from Assumption~\ref{a: ergodicity} and since $\|\tG(Q_n,g_n, z_n)\|_\infty\leq 4 e^{\frac{\theta}{\gamma}\Cm},$ and $(d)$ follows from definition of $\taun$ and the inequality $(p+q)^2\leq 2(p^2+q^2).$

Hence, 
\begin{multline}\label{e: bound.T1}
    \alpha_n\bE T_{n,1}
    \leq 2\alpha^2_n\bE\|y_n\|^2_2 
    \\
    + 2\alpha^2_n\bE\|y_n - y_{n-\taun}\|^2_2 + 16SA\ e^{\frac{2\theta}{\gamma}\Cm} \alpha^4_n.
\end{multline}
Next, we bound $T_{n,2}$ as
\begin{align}\label{e: bound.T2}
    & \alpha_n\bE T_{n,2}
    \nonumber\\
    & = 2\bE \big\langle(I-\alpha_nD_\mu)(y_n-y_{n-\taun}), \tG(Q_n,g_n, z_n) \big\rangle
    \nonumber\\
    & \overset{(a)}{\leq} \bE \|y_n-y_{n-\taun}\|_2^2 + \alpha^2_n\bE \|\tG(Q_n,g_n, z_n)\|^2_2
    \nonumber\\
    & \overset{(b)}{\leq} \bE \|y_n-y_{n-\taun}\|_2^2 + 16SA\alpha^2_n e^{\frac{2\theta}{\gamma}\Cm}
\end{align}
where $(a)$ uses Young's inequality and $\|I-\alpha_nD_\mu\|_2\leq 1,$ $(b)$ follows since $\|\tG(Q_n,g_n,z_n)\|_\infty\leq 4e^{\frac{2\theta}{\gamma}\Cm}.$

Lastly, we bound $\alpha_n\bE T_{n,3}$ as
\begin{align}\label{e: bound.T3}
    & \alpha_n\bE T_{n,3}
    \nonumber\\
    & = 2\alpha_n\bE\big\langle(I-\alpha_nD_\mu)(y_{n-\taun}),
    \nonumber\\
    & \qquad \times \big(\tG(Q_n,g_n, z_n) - \tG(Q_{n-\taun},g_{n-\taun}, z_n) \big)\big\rangle
    \nonumber\\
    & \overset{(a)}{\leq} \alpha^2_n\ \bE\|y_{n-\taun}\|^2_2 
    \nonumber\\
    & \qquad + \bE\|\tG(Q_n,g_n, z_n) - \tG(Q_{n-\taun},g_{n-\taun}, z_n)\|^2_2
    \nonumber\\
    & \overset{(b)}{\leq} \alpha^2_n\ \bE\|y_{n-\taun}\|^2_2 + \bE\|g_n - g_{n-\taun}\|^2_2
    \nonumber\\
    & \qquad  + SA\theta^2\left( e^{2\theta}+e^{\frac{2\theta}{\gamma}\Cm}\right)\bE\|Q_n-Q_{n-\taun}\|^2_2
    \nonumber\\
    & \overset{(c)}{\leq} \alpha^2_n\ \bE\|y_n\|^2_2 + \alpha^2_n\ \bE\|y_n - y_{n-\taun}\|^2_2 
    \nonumber\\
    & \qquad + \bE\|g_n - g_{n-\taun}\|^2_2
    \nonumber\\
    & \qquad  + SA\theta^2\left( e^{2\theta}+e^{\frac{2\theta}{\gamma}\Cm}\right)\bE\|Q_n-Q_{n-\taun}\|^2_2
\end{align}
where $(a)$ uses Young's inequality, and $(c)$ uses the inequality $(p+q)^2\leq 2(p^2+q^2).$

Now, since $Q_k, \frac{\gamma}{\theta}\ln g_k \in \cK'$ for all $k,$ we have
\begin{align}\label{e: diff.Q}
    \|Q_n-Q_{n-\taun}\|_\infty 
    & \overset{(a)}{\leq} \sum_{k=n-\taun}^{n-1}\| Q_{k+1} - Q_k\|_\infty
    \nonumber\\
    & \overset{(b)}{\leq} \sum_{k=n-\taun}^{n-1} \beta_k\left\| -\frac{\gamma}{\theta}\ln g_k - Q_k \right\|_\infty
    \nonumber\\
    &  \overset{(c)}{\leq} 2\Cm\sum_{k=n-\taun}^{n-1} \beta_k
    \nonumber\\
    &  \overset{(d)}{\leq} 2\Cm\taun \beta_{n-\taun} \nonumber\\
    & \overset{(e)}{\leq} 4\Cm\taun \beta_n,
\end{align}
where $(a)$ follows from triangle inequality, $(b)$ follows from~\eqref{e: 2TS}, $(c)$ follows from~\eqref{e: Q.bound}, and $(d)$ follows since $(\alpha_n)$ is non-increasing. Lastly, $(e)$ follows since $\beta_{n-\taun}<2\beta_n$ for $n>\nstr\geq \nstr^{(3)}.$

Likewise,
\begin{align}\label{e: diff.g}
    \|g_n-g_{n-\taun}\|_\infty 
    & \overset{(a)}{\leq} \sum_{k=n-\taun}^{n-1}\| g_{k+1} - g_k\|_\infty
    \nonumber\\
    &  \overset{(b)}{\leq} \sum_{k=n-\taun}^{n-1} \alpha_k\left| \hG(Q_k, z_k) - g_k(s_k,a_k) \right|_\infty
    \nonumber\\
    &  \overset{(c)}{\leq} 2e^{\frac{\theta}{\gamma}\Cm}\sum_{k=n-\taun}^{n-1} \alpha_k
    \nonumber\\
    &  \overset{(d)}{\leq} 2e^{\frac{\theta}{\gamma}\Cm}\taun \alpha_{n-\taun} 
    \nonumber\\
    & \overset{(e)}{\leq} 4e^{\frac{\theta}{\gamma}\Cm}\taun \alpha_n,
\end{align}
where $(a)$ follows from triangle inequality, $(b)$ follows from~\eqref{e: 2TS}, $(c)$ follows from~\cite[Equation 26]{thoppe2026reinforcement}, and $(d)$ follows since $(\alpha_n)$ is non-increasing. Lastly, $(e)$ follows since $2\alpha_{n-\taun}<\alpha_n$ for $n>\nstr\geq \nstr^{(3)}.$

Moreover, since $u\mapsto e^{-\frac{\theta}{\gamma}u}$ is $\left(\frac{\theta}{\gamma} e^{\frac{\theta}{\gamma}\Cm}\right)$-Lipschitz on $[-\Cm,\Cm],$ we have
\begin{align}\label{e: diff.y}
    & \|y_n-y_{n-\taun}\|_\infty 
    \nonumber\\
    & \leq  \|g_n -g_{n-\taun}\|_\infty + \frac{\theta}{\gamma} e^{\frac{\theta}{\gamma}\Cm}\|TQ_n-TQ_{n-\taun}\|_\infty
    \nonumber\\
    & \overset{(a)}{\leq}  \|g_n -g_{n-\taun}\|_\infty + \theta e^{\frac{\theta}{\gamma}\Cm}\|Q_n-Q_{n-\taun}\|_\infty
    \nonumber\\
    & \overset{(b)}{\leq} 4e^{\frac{\theta}{\gamma}\Cm}\ \taun \alpha_n + 4\theta e^{\frac{\theta}{\gamma}\Cm}\Cm\ \taun \beta_n,
\end{align}
where $(a)$ follows since $T$ is a $\gamma$-contraction w.r.t $\|\cdot\|_\infty,$ and $(b)$ follows from~\eqref{e: diff.Q} and~\eqref{e: diff.g}.

To bound $\alpha_n\bE T_n,$ we combine~\eqref{e: bound.T1},~\eqref{e: bound.T2},~\eqref{e: bound.T3},~\eqref{e: diff.Q},~\eqref{e: diff.g}, and~\eqref{e: diff.y}, and use $\|\cdot\|_2 \leq \sqrt{SA}\|\cdot\|_\infty$ to get
\begin{align}\label{e: bound.Tn}
    \alpha_n\bE T_n & \leq 3\alpha^2_n\bE\|y_n\|^2_2 
    \\
    & + 8SA e^{\frac{2\theta}{\gamma}\Cm}\ \taun^2 \alpha^2_n + 8SA\theta^2 e^{\frac{2\theta}{\gamma}\Cm}\Cm^2\ \taun^2 \beta^2_n 
    \nonumber\\
    & + 32SA\ e^{\frac{2\theta}{\gamma}\Cm} \alpha^2_n 
    \nonumber\\
    & + SA e^{\frac{2\theta}{\gamma}\Cm}\taun^2 \alpha^2_n
    \nonumber\\
    & + (SA)^2\theta^2\left( e^{2\theta}+e^{\frac{2\theta}{\gamma}\Cm}\right)\Cm^2\taun^2 \beta^2_n
    \nonumber\\
    & \leq  3\alpha^2_n\bE\|y_n\|^2_2 + C_{2,1}\taun^2 \alpha^2_n,
\end{align}
where we use $(SA)\leq (SA)^2,$ $1\leq \taun^2$ and $\beta^2_n\leq \alpha^2_n$ for the last inequality, and define
\begin{multline*}
    C_{2,1} 
    := (SA)^2
    \\
    \times\left[41e^{\frac{2\theta}{\gamma}\Cm}\taun^2 \alpha^2_n
    +  \theta^2\Cm^2\left(e^{2\theta}+2e^{\frac{2\theta}{\gamma}\Cm}  \right)\right].
\end{multline*}

Finally, plugging in~\eqref{e: bound.Tn} into~\eqref{e: yn.intermediate} gives us
\begin{align}\label{e: yn.one.step}
    \bE\|y_{n+1}\|^2_2 
    & \overset{(a)}{\leq} \big(1-2\delta\alpha_n + 5\alpha^2_n\big)\|y_n\|^2_2
    + C_2\taun^2 \alpha^2_n
    \nonumber\\
    & \overset{(a)}{\leq}  \big(1-\delta\alpha_n \big)\|y_n\|^2_2
    + C_2\taun^2 \alpha^2_n,
\end{align}
where $(a)$ follows as we take $1\leq \taun^2$ and use the definition of $C_2$ from Table~\ref{tab:constants-2TS}.
%
% \begin{align*}
%     %
%     C_2: = C_{2,1} + (4SA)\ e^{\frac{2\theta}{\gamma}\Cm}(8 + 3\Cm^2 ),
%     %
% \end{align*}
%
while $(b)$ follows since $5\alpha^2_n\leq \delta\alpha_n,$ for $n\geq \nstr\geq \nstr^{(2)}.$ 

We are now ready to bound the convergence rate for $(y_n).$ Applying~\eqref{e: yn.one.step} recursively, we have for $N>\nstr,$
\begin{align}
    \bE\|y_N\|^2_2 & \leq \Galpha_{\nstr:N}\ y_{\nstr} + C_2\sum_{n=\nstr}^{N-1}\taun^2\alpha^2_n\ \Galpha_{n+1:N}
    \nonumber\\
    & \leq \Galpha_{\nstr:N}\ y_{\nstr} + \tau_N^2C_2\sum_{n=\nstr}^{N-1}\alpha^2_n\ \Galpha_{n+1:N},
\end{align}
where for $\nstr\leq n_1<n_2,$ we define
\begin{equation}
    \Galpha_{n_1:n_2}:= \prod_{n=n_1}^{n_2-1}(1-\delta\alpha_n).
\end{equation}
Then, using Lemma~\ref{lem: Gnn.bound}, we have for $N>\nstr,$
\begin{align}
    & \bE\|y_N\|^2_2 
    \nonumber\\
    & \leq \bE\|y_{\nstr} \|^2_2\left(\frac{\nstr+1}{N+1} \right)^{c\ln N}
    +  8C_2\frac{\tau_N^2\ln N}{\delta N}
    \nonumber\\
    & \leq 2e^{\frac{2\theta}{\gamma}\Cm}\left( 1+ \theta^2\Cm^2\right)\left(\frac{\nstr+1}{N+1} \right)^{c\ln N}
    + 8C_2\frac{\tau_N^2\ln N}{\delta N}
    \nonumber\\
    & \overset{(a)}{\leq } C'_y \frac{\tau^2_N\ln N}{N},
\end{align}
where $C'_y$ is as defined in Table~\ref{tab:constants-2TS}.

We use $\tau_N \leq \ln N[2\ln C + 4]/\ln(1/\rho)$ and the definition of $C_g$ from Table~\ref{tab:constants-2TS} to get
\begin{align}\label{e: rate.gn}
    \bE\|y_N\|^2_\infty \leq \bE\|y_N\|^2_2 \leq C^2_g\frac{\ln^3 N}{N},
\end{align}
The claim follows for $(y_N)$ by Jensen's inequality.

Finally, we use the convergence rates derived in~\eqref{e: rate.gn} to bound the slower updates, i.e., $(Q_n).$ Since the faster iterate $g_n$ tracks a moving target, this results in an additional tracking error.

Recall from~\eqref{e: 2TS}, we can write the $(Q_n)$ update as
\begin{align}
    (Q_{n+1} - \Qstr) 
    & = (Q_n - \Qstr) + \beta_n\left[ TQ_n - Q_n\right] + \beta_n\epsp_n
    \nonumber\\
    & = (1-\beta_n)(Q_n - \Qstr) 
    \nonumber\\
    & \qquad+ \beta_n\left[ TQ_n - T\Qstr\right] + \beta_n\epsp_n,
\end{align}
where we define the tracking error $\epsp_n:=-\frac{\gamma}{\theta}\ln g_n - TQ_n.$ Since, $T$ is a $\gamma$-contraction w.r.t $\|\cdot\|_\infty,$ we get
\begin{align}\label{e: desired.Q1}
    & \| Q_{n+1} - \Qstr\|_\infty 
    \nonumber\\
    &  \leq  (1-\alpha_n)\|Q_n - \Qstr\|_\infty 
    \nonumber\\
    & \quad+ \beta_n\left \|TQ_n - T\Qstr\right\|_\infty + \beta_n\|\epsp_n\|_\infty
    \nonumber\\
    &\quad\leq (1-(1-\gamma)\alpha_n)\|Q_n - \Qstr\|_\infty + \beta_n\|\epsp_n\|_\infty.
\end{align}
Using the fact that on any interval $[c,d]$ with $c>0,$ the function $u\mapsto \ln(u)$ is $(1/c)$-Lipschitz and $g_n, e^{-\frac{\theta}{\gamma}TQ_n} \in [\exp(-\theta\Cm/\gamma), \exp(\theta\Cm/\gamma)]^{SA},$ we bound the tracking error term in~\eqref{e: desired.Q1} as follows:
\begin{align}
    \|\epsp_n\|_\infty & = \left\|\frac{\gamma}{\theta}\ln g_n + TQ_n\right\|_\infty 
    = \frac{\gamma}{\theta}\left\|\ln g_n - \ln e^{ -\frac{\theta}{\gamma}TQ_n} \right\|_2 
    \nonumber\\
    & \leq \frac{\gamma}{\theta} e^{\frac{\theta}{\gamma}\Cm}\|g_n - e^{-\frac{\theta}{\gamma}TQ_n}\|_2.
\end{align}
Plugging the above bound into~\eqref{e: desired.Q} and taking expectation gives for $n\geq \nstr.$
\begin{align}\label{e: desired.Q}
    \bE\| Q_{n+1}& - \Qstr\|_\infty 
    \nonumber\\
    & \leq (1-(1-\gamma)\beta_n)\bE\|Q_n - \Qstr\|_\infty
    \nonumber\\
    & + \beta_n \frac{\gamma}{\theta} e^{\frac{\theta}{\gamma}\Cm}\bE\|y_n\|_2
    \nonumber\\
    &\leq (1-(1-\gamma)\beta_n)\bE\|Q_n - \Qstr\|_\infty 
    \nonumber\\
    & +  \frac{\gamma \sqrt{SA}}{\theta} e^{\frac{\theta}{\gamma}\Cm}{C_g}\beta_n \frac{\ln^{3/2} n}{\sqrt n}
    \nonumber\\
    & = (1-(1-\gamma)\beta_n)\bE\|Q_n - \Qstr\|_\infty 
    \nonumber\\
    & +  C_3 \beta_n \frac{\ln^{3/2} n}{\sqrt n},
\end{align}
where $C_3$ is as defined in Table~\ref{tab:constants-2TS}. Applying the above one-step bound recursively gives us for all $N>\nstr,$
\begin{align}
    \bE\|Q_N & -\Qstr\|_\infty 
    \nonumber\\
    &\leq 
    \Gbeta_{\nstr:N}\ \bE\|Q_{\nstr}-\Qstr\|_\infty 
    \nonumber\\
    & \qquad + C_3\ln^{3/2}\ln N\sum_{n=\nstr}^{N-1}\beta^2_n\  \Gbeta_{n+1:N}
    \nonumber\\
    & \leq 2\Cm\ \Gbeta_{\nstr:N}
    \nonumber\\
    & \qquad + C_3\ln^{3/2}\ln N\sum_{n=\nstr}^{N-1}\beta^2_n\  \Gbeta_{n+1:N},
\end{align}
where 
\begin{align}
    \Gbeta_{n_1:n_2}:= \prod_{n=n_1}^{n_2-1}\big(1-(1-\gamma)\beta_n \big).
\end{align}
Then, using Lemma~\ref{lem: Gnn.bound}, we have 
\begin{align*}
    &\bE\|Q_N-\Qstr\|_\infty
    \nonumber\\
    & \leq 2\Cm\left(\frac{\nstr+1}{N+1}\right)^{b(1-\gamma)}
    + C_3C_{\Gamma,1}(1-\gamma,\beta)\frac{\ln^{3/2}}{\sqrt{N+1}}
    \nonumber\\
    & \leq C_Q^2\frac{\ln^{3/2}}{\sqrt{N+1}},
\end{align*}
where the last inequality uses $1\leq \ln^{3/2}N$ and the definition of $C_Q$ from Table~\ref{tab:constants-2TS}. 

This completes the proof of Theorem~\ref{thm: 2TS}. \hfill\qed

\section{Technical Lemmas}
\label{appendix: tech}

\begin{lemma}\label{lem: Gnn.bound}
    Given step size sequence $(\alpha_n),$ define for all $\nstr\leq n_1<n_2:$
    \[
        \Galpha_{n_1:n_2}:=\prod_{n=n_1}^{n_2-1}\left( 1- c\alpha_n\right),
    \]
    where $c>0,$ is some constant such that $c\alpha_n<1,$  for $n\geq \nstr.$ Then, the following holds for all $\nstr\leq k<N:$
    \begin{enumerate}
        \item When $\alpha_n=\frac{1}{(n+1)^\alpha},$ with $\alpha\in (1/2,1),$ we have
        \begin{multline*}
            \Galpha_{k:N} \leq e^{\frac{c}{1-\alpha}(k+1)^{1-\alpha}} e^{-\frac{c}{1-\alpha}N^{1-\alpha} }
            \\
            \quad \text{and} \quad \sum_{n=k}^{N-1}\alpha^2_n \ \Galpha_{n+1:N} \leq C_{\Gamma}(c,\alpha)\ \alpha_N.
        \end{multline*}
        \item When $\alpha_n=\frac{a}{n+1},$ with $ac>1,$ we have
        \begin{multline*}
            \Galpha_{k:N} \leq \left(\frac{k+1}{N+1}\right)^{ac}
            \\
            \quad \text{and} \quad \sum_{n=k}^{N-1}\alpha^2_n \ \Galpha_{n+1:N} \leq \frac{C_{\Gamma,1}(c,a)}{N+1}.
        \end{multline*}
        \item When $\alpha_n=\frac{\ln n}{n+1},$ we have for $k>e^{1/c},$
        \begin{multline*}
            \Galpha_{k:N} \leq \left(\frac{k+1}{N+1} \right)^{\frac{c\ln N}{2}}
            \\
            \quad \text{and}\quad  \sum_{n=k}^{N-1}\alpha^2_n \ \Galpha_{n+1:N} \leq  \frac{16\e^{c/e}ln N}{cN}.
        \end{multline*}
    \end{enumerate}
    Moreover, let $\alpha_n = \frac{1}{(n+1)^\alpha},$ with $\alpha\in (1/2,1),$ define 
    \begin{multline}
        G_{n_1:n_2} := \prod_{n=n_1}^{n_2} \big[ I - \alpha_n (I-B) \big],
        \\
        \text{and} \quad S_{n_1:n_2} := \alpha_{n_1}\sum_{n=n_1+1}^{n_2-1} G_{n_1+1:n}.
    \end{multline}
\end{lemma}
where $B$ is a matrix satisfying $\|B\|_\infty\leq c_B<1.$  Then, for we have for all $\nstr\leq k<N$
\begin{multline*}
        \|G_{k:N}\|_\infty \leq e^{\frac{1-c_B}{1-\alpha}(k+1)^{1-\alpha}} e^{-\frac{1-c_B}{1-\alpha}N^{1-\alpha} }
        \\
        \text{and} \quad \sum_{n=k}^{N-1}\alpha^2_n \|G_{n+1:N}\|_\infty \leq C_{\Gamma}(1-c_B,\alpha)\ \alpha_N,
\end{multline*}
Furthermore, $\|S_{k:N}\| \leq K_G,$ for some constant $K_G>0.$
\begin{proof}[\textbf{Proof of Lemma~\ref{lem: Gnn.bound}}]
    Using the fact that $1-x<e^{-x},$ we can write
    \begin{align}\label{e: uppe.Gnn}
        \Galpha_{k:N} \leq e^{-c \sum_{n=k}^{N-1}\alpha_n}.
    \end{align}
    \paragraph{Case 1:} Taking $\alpha_n=\frac{1}{(n+1)^\alpha},$ with $\alpha\in (1/2,1)$ gives us
    \begin{align}
        \Galpha_{k:N} & \leq e^{{-c \sum_{n=k}^{N-1}\frac{1}{(n+1)^\alpha}}}
        \nonumber\\
        & \leq e^{{-c \int_{k+1}^{N+1}\frac{1}{s^\alpha}ds }}
        \nonumber\\
        & = e^{{-\frac{c}{1-\alpha}[(N+1)^{1-\alpha} - (k+1)^{1-\alpha}] }}
        \nonumber\\
        & \leq e^{\frac{c}{1-\alpha}(k+1)^{1-\alpha}}\ e^{-\frac{c}{1-\alpha}N^{1-\alpha} }.
    \end{align}
    For the next claim, we use~\eqref{e: uppe.Gnn} to get
    \begin{multline}\label{e: Galpha.sum.split}
        \sum_{n=k}^{N-1}\alpha^2_n\ \Galpha_{n+1:N} 
        \leq \sum_{n=k}^{N-1}\alpha^2_n\ e^{-c\sum_{m=n+1}^{N-1}\alpha_m }
        \\
        \leq \left[ \max_{k\leq n<N} f_n\ \right]  \sum_{n=k}^{N-1}\alpha_n\ e^{-\frac{c}{2}\sum_{m=n+1}^{N-1}\alpha_m},
    \end{multline}
    where we define
    \begin{equation}
        f(n):= \alpha_n\ e^{-\frac{c}{2}\sum_{m=n+1}^{N-1}\alpha_m}.
    \end{equation}
    To bound the factor $\max_{k\leq n<N} f_n,$ we consider two cases. First, let $n\geq \ceil{N/2}.$ Then, we have
    \begin{align}
        \frac{f_n}{\alpha_N} \overset{(a)}{\leq} \frac{\alpha_n}{\alpha_N} =
    \left(\frac{N+1}{n+1}\right)^\alpha  \overset{(b)}{\leq} 2^\alpha,
    \end{align}   
    where $(a)$ follows since $e^{-\frac{c}{2}\sum_{m=n+1}^{N-1}\alpha_m}\leq 1,$ and $(b)$ follows since $(n+1)>\frac{N}{2}+1>\frac{N+1}{2}.$
    
    Next, consider the case when $n\leq \floor{N/2}.$ Then, we have
    \begin{align*}
        \sum_{m=n+1}^{N-1}\alpha_m \overset{(a)}{\geq} \sum_{m=\floor{N/2}}^{N-1}\alpha_m \overset{(b)}{\geq} \sum_{m=\floor{N/2}}^{N-1}\alpha_{N-1} > \frac{N}{2}\alpha_{N-1},
    \end{align*}
    where $(a)$ follows since for all $m,$ $\alpha_m>0,$ and $(b)$ follows since $(\alpha_m)$ is decreasing. Consequently, we get
    \begin{align}
        \frac{f_n}{\alpha_N} & \overset{(a)}{\leq} \frac{1}{\alpha_N}\ e^{-\frac{c}{2}\sum_{m=n+1}^{N-1}\alpha_m}
        \leq (N+1)^\alpha\ e^{-\frac{c}{4}N\alpha_{N-1}}
        \nonumber\\
        & \overset{(b)}{<} 2^\alpha \left( N^\alpha\ e^{-\frac{c}{4} N^{1-\alpha}} \right)
        \overset{(c)}{\leq} \left[\frac{4e\alpha}{c(1-\alpha)}\right]^{\frac{\alpha}{1-\alpha}},
    \end{align}
    where $(a)$ uses $\alpha_n<1,$ and $(b)$ uses $N+1<2N.$ Finally, $(c)$ follows as we can show using calculus that $y\mapsto y^\alpha\exp\left(-cy^{(1-\alpha)} \right)$ attains maximum at $y=\big[\frac{\alpha}{c(1-\alpha)}\big]^{1/(1-\alpha)}.$ Hence, we conclude that
    \begin{align}\label{e: bound.factor}
        \left[\max_{k\leq n\leq N} f_n \right]\leq \alpha_N\cdot \max\left\{2^\alpha, \left[\frac{4 e\alpha}{c(1-\alpha)}\right]^{\frac{\alpha}{1-\alpha}} \right\}.
    \end{align}
    To bound the sum in~\eqref{e: Galpha.sum.split}, define $s_n = \sum_{m=0}^{n}\alpha_m$ and $\Delta_n:=[s_n - s_{n-1}].$ Then, for all $0\leq k<N,$
    \begin{align}\label{e: alpha.sum.B}
        \sum_{n=k}^{N-1} \alpha_ne^{-\frac{c}{2}\sum_{m=n+1}^{N-1}\alpha_m }\leq \sum_{n=0}^{N-1} e^{-\frac{c}{2}[s_{N-1} - s_n]}\Delta_n
    \end{align}
    Treating the r.h.s as a Riemann sum, we have
    \begin{align}\label{e: alpha.sum.Riemann}
        \sum_{n=0}^{N-1}&e^{- \frac{c}{2}[s_{N-1} - s_n]}\cdot \Delta_n
        \leq \int_{0}^{s_{N-1}}e^{\frac{c}{2}(s-s_{N-1})}\ ds 
        \nonumber\\
        &= \frac{2}{c}\left(1 - e^{(s_0 - s_{N-1})}\right) \leq \frac{2}{c}.
    \end{align}
    The claim follows for \textbf{Case 1} as we combine~\eqref{e: Galpha.sum.split},~\eqref{e: bound.factor}, and~\eqref{e: alpha.sum.Riemann}.
    
    \paragraph{Case 2:} Taking $\alpha_n=\frac{a}{n+1}$ in~\eqref{e: uppe.Gnn} gives us
    \begin{align}\label{e: G.lin.bound}
        \Galpha_{k:N} & \leq e^{-ac \sum_{n=k}^{N-1}\frac{1}{n+1}}
        \nonumber\\
        & \leq e^{{-ac \int_{k+1}^{N+1}\frac{1}{s}ds }}
        \nonumber\\
        & = e^{{-ac\log(\frac{N+1}{k+1})}} = \left(\frac{k+1}{N+1}\right)^{ac}.
    \end{align}
    Next, using~\eqref{e: G.lin.bound}, we have
    \begin{align}
        \sum_{n=k}^{N-1}\alpha^2_n \Galpha_{n+1:N} 
        & \leq \sum_{n=k}^{N-1}\frac{a^2}{(n+1)^2}\left(\frac{n+2}{N+1}\right)^{ac}
        \nonumber\\
        & \overset{(a)}{\leq} \frac{a^2 2^{ac}}{(N+1)^{ac}} \sum_{n=k}^{N-1} (n+1)^{ac-2}
        \nonumber\\
        & \overset{(b)}{\leq} \frac{a^2 2^{ac}}{(N+1)^{ac}} \int_{k+1}^{N+1} s^{ac-2}ds
        \nonumber\\
        & \overset{(c)}{\leq} \frac{a^2 2^{ac}}{(N+1)^{ac}} \int_{0}^{N+1} s^{ac-2}ds
        \nonumber\\
        & = \frac{a^2 2^{ac}}{(N+1)^{ac}}\frac{(N+1)^{ac-1}}{(ac-1)}
        \nonumber\\
        & = \frac{a^2 2^{ac}}{(ac-1)}\frac{1}{(N+1)}.
    \end{align}
    The claim thus follows for \textbf{Case 2}.

    \paragraph{Case 3:} When $\alpha_n=\frac{\ln (n+1)}{n+1},$~\eqref{e: uppe.Gnn} gives us
    \begin{align}\label{e: G.log.lin.bound}
        \Galpha_{k:N} & \leq e^{-c \sum_{n=k}^{N-1}\frac{\ln(n+1)}{n+1}}
        \nonumber\\
        & \leq e^{{-c \int_{k+1}^{N+1}\frac{\ln s}{s}ds }}
        \nonumber\\
        & = e^{\frac{c}{2}[\ln(k+1))^2 - (\ln(N+1))^2] } 
        \nonumber\\
        & \overset{(a)}{=} \frac{(k+1)^\frac{c\ln (k + 1)}{2}}{(N+1)^\frac{c\ln (N + 1)}{2}} \leq \left(\frac{k+1}{N+1} \right)^\frac{c\ln (N)}{2}.
    \end{align}
    Then, using~\eqref{e: G.log.lin.bound}, we get
    \begin{align}
    \sum_{n=k}^{N-1}\alpha_n^2\Galpha_{n+1:N}
    &\leq
    \sum_{n=k}^{N-1}
    \frac{(\ln n)^2}{(n+1)^2}
    \left(\frac{n+2}{N+1}\right)^{\frac{c\ln N}{2}}
    \nonumber\\
    &\leq
    \frac{4\ln^2 N}{(N+1)^{\frac{c\ln N}{2}}}
    \sum_{n=k}^{N-1}
    (n+2)^{\frac{c\ln N}{2}-2}
    \nonumber\\
    &\leq
    \frac{4\ln^2 N}{(N+1)^{\frac{c\ln N}{2}}}
    \int_{k+2}^{N+2}
    s^{\frac{c\ln N}{2}-2}\,ds
    \nonumber\\
    &\leq
    \frac{4\ln^2 N}{(N+1)^{\frac{c\ln N}{2}}}
    \frac{(N+2)^{\frac{c\ln N}{2}-1}}
    {\frac{c\ln N}{2}-1}
    \nonumber\\
    &\leq
    \frac{4e^{c/2}\ln^2 N}
    {(N+2)\left(\frac{c\ln N}{2}-1\right)}
    \nonumber\\
    &\leq
    \frac{16e^{c/e}\ln N}{cN},
\end{align}

    At last, since $\|B\|_\infty \leq c_B,$ the matrix product $G_{k:N}$ satisfies
    \begin{equation}
        G_{k:N} \leq \prod_{n=k}^{N-1}(1-(1-c_B)\alpha_n).
    \end{equation}
    Hence, the claim follows using previous arguments.
    
    This completes the proof of Lemma~\ref{lem: Gnn.bound}.
\end{proof}

\end{document}